\documentclass[11pt]{article}
\usepackage{bbm}
\usepackage{eqnarray,amsmath,amsfonts,amsthm,mathrsfs}
\usepackage{color}
\usepackage{bm}
\usepackage{amssymb}
\usepackage{epsfig}
\usepackage{epsf}
\usepackage{float}
\usepackage{subfigure}
\usepackage{amsfonts,amsmath,amsthm,amssymb,graphicx,float,fancyhdr,multirow,hyperref}
\usepackage{booktabs,longtable,authblk}
\usepackage{mathrsfs,hhline}
\usepackage{enumitem}

\usepackage{fancyhdr}
\usepackage[top=2.5cm, bottom=2.5cm, left=3cm, right=3cm]{geometry}
\usepackage{graphicx}
\newtheorem{theorem}{Theorem}
\newtheorem{assumption}{Assumption}
\newtheorem{corollary}{Corollary}
\newtheorem{definition}{Definition}

\newtheorem{lemma}{Lemma}[section]
\newtheorem{proposition}{Proposition}[section]
\newtheorem{remark}{Remark}[section]
\newtheorem{example}{Example}[section]

\allowdisplaybreaks[4]

\def\begeqn{\begin{equation}}
\def\endeqn{\end{equation}}
\def\begth{\begin{theorem}}
\def\endth{\end{theorem}}
\def\begprop{\begin{proposition}}
\def\endprop{\end{proposition}}
\def\begcor{\begin{corollary}}
\def\endcor{\end{corollary}}
\def\begdef{\begin{definition}}
\def\enddef{\end{definition}}
\def\beglemm{\begin{lemma}}
\def\endlemm{\end{lemma}}
\def\begexm{\begin{example}}
\def\endexm{\end{example}}
\def\begrem{\begin{remark}}
\def\endrem{\end{remark}}
\def\begassum{\begin{assumption}}
\def\endassum{\end{assumption}}

\newtheorem{prop}{Proposition}
\newtheorem{cor}{Corollary}
\newtheorem{rem}{Remark}

\title{Online Regularized Learning Algorithm for Functional Data$^\dag$\footnotetext{\dag~The work described in this paper is supported by Zhejiang Provincial Natural Science Foundation of China [Project No.\ LR20A010001], National Natural Science Foundation of China [Project No.\ U21A20426, No.\ 12271473], and Fundamental Research Funds for the Central Universities [Project No.\ 2021XZZX001]. Email addresses: ymao@zju.edu.cn (Y. Mao), guozhengchu@zju.edu.cn (Z. C. Guo). The corresponding author is Zheng-Chu Guo.}}
\author{Yuan Mao and Zheng-Chu Guo\\
\small  School of Mathematical Sciences, Zhejiang University, Hangzhou 310058, P. R. China \\ }
\date{}
\begin{document}

\maketitle
\begin{abstract}
    In recent years, functional linear models have attracted growing attention in statistics and machine learning, with the aim of recovering the slope function or its functional predictor.
    This paper considers online regularized learning algorithm for functional linear models in reproducing kernel Hilbert spaces.
    Convergence analysis of excess prediction error and estimation error are provided with polynomially decaying step-size and constant step-size, respectively.
    Fast convergence rates can be derived via a capacity dependent analysis.
    By introducing an explicit regularization term, we uplift the saturation  boundary of unregularized online learning algorithms when the step-size decays polynomially, and establish fast convergence rates of estimation error without capacity assumption. However, it remains an open problem to obtain capacity independent convergence rates for the estimation error of the unregularized online learning algorithm with decaying step-size.
    It also shows that convergence rates of both prediction error and estimation error with constant step-size are competitive with those in the literature.
\end{abstract}

{\bf Keywords}: Functional linear model, Online learning, Reproducing kernel Hilbert spaces, Regularization.

\section{Introduction}
\label{sec:intro}
    In this paper, we consider the following functional linear regression model
    \begin{equation}
        \label{eq:linearmodel}
        Y = \alpha^{*} + \int_{\mathcal{T}} \beta^{*}(t) X(t) dt + \varepsilon,
    \end{equation}
    which aims to find the relationship between the scalar response $ Y $ and a random function $ X $.
    Here the domain $ \mathcal{T} $ of $ X $ is a compact subset of an Euclidean space, $ \alpha^{*} $ is a constant intercept, $ \beta^{*}: \mathcal{T} \to \mathbb{R} $ is an unknown slope function, and $ \varepsilon $ is a random noise with zero mean value and finite variance $ \sigma^2 < \infty $ which is independent of $X.$
    Let $ ( \mathcal{L}^2 (\mathcal{T}),  \left<\cdot,\cdot \right>_{2},\| \cdot \|_{2} ) $ be the Hilbert space of square-integrable functions from $ \mathcal{T} $ to $ \mathbb{R} $, and we assume that $ X , \beta^{*} \in \mathcal{L}^2 (\mathcal{T}) $.
    For the sake of simplicity, we assume $ \alpha^{*} = 0 $ and $ \mathbb{E} [X] = 0 $ which leads to $ \mathbb{E}[Y] = 0 $.

    We analyze two types of learning tasks for functional liner regression model (\ref{eq:linearmodel}) in this paper including estimation problem and prediction problem.
    The estimation problem aims at estimating the slope function $ \beta^{*} $.
    And the goal of prediction problem is to recover the unknown functional $ \eta^{*} $ on $ \mathcal{L}^2 (\mathcal{T}) $ defined by
    \begin{equation*}
        \eta^{*}:X\mapsto  \left< \beta^{*}, X \right>_{2} = \int_{\mathcal{T}} \beta^{*}(t) X(t) dt.
    \end{equation*}
    It is obvious that $ \eta^{*} $ is determined by $ \beta^{*} $ which introduces a smoothing effect and leads to an easier learning problem \cite{Chen2022}.

    There is a large literature on functional linear regression model (\ref{eq:linearmodel}), see \cite{Ramsay2005,Cai2006,Hall2007,Cai2012,Fan2019,Yuan2010,Chen2022,guoCapacityDependentAnalysis2022,tong2021distributed} and references therein.
    One of the commonly used tools for the regression problem is the functional principal component analysis (FPCA) \cite{Ramsay2005,Cai2006,Hall2007}, where the slope function $ \beta^{*} $ is estimated by a linear combination of the estimated eigenfunctions of covariance function of the random function $X$.
    Another line of research employs the kernel methods to establish the estimator of $ \beta^{*} $ under the assumption that $ \beta^{*} $ belongs to a \emph{reproducing kernel Hilbert space} (RKHS) associated with a reproducing kernel $ K $ \cite{Cai2012,Fan2019,Yuan2010}, and the estimator of $ \beta^{*} $ can be represented by a linear combination of kernel functions.
    Recently, to cope with the challenges of memory bottlenecks and algorithm scalability arising from massive datasets, distributed learning methods \cite{tong2021distributed} and stochastic gradient descent (SGD) methods \cite{Chen2022,guoCapacityDependentAnalysis2022} are studied in the framework of  functional linear regression model (\ref{eq:linearmodel}) and RKHS.

    The algorithm studied in this paper is implemented in an RKHS $ (\mathcal{H}_{K}, \| \cdot \|_{K}) $ associated with a Mercer kernel $ K:\mathcal{T} \times \mathcal{T} \to \mathbb{R} $, which means that $ K $ is a continuous, symmetric and positive semi-definite function on $ \mathcal{T} \times \mathcal{T} $.
    The RKHS $ \mathcal{H}_{K} $ is defined to be the completion of the linear span of functions $ \{ K_{t}(\cdot) = K(t,\cdot): t \in \mathcal{T} \} $ with the inner product $ \left< K_{t},K_{t^\prime} \right>_{K} = K(t,t^\prime ) $ and the \emph{reproducing property}
    \begin{equation}
        f(t) = \left< f, K_{t} \right>_{K}, \text{ for any } t \in \mathcal{T} \text{ and } f \in \mathcal{H}_{K}.
    \end{equation}
    Since the distribution of $ Z = (X,Y) $ is unknown, given a training sample $ \bm{Z} = \{Z_{i} = (X_{i},Y_{i})\}_{i=1}^{n} \subset Z^{n} $ of $ n $ independent copies of $ (X,Y) $, classical batch learning algorithm aims to find an estimator of $ \beta^{*} $ in $ \mathcal{H}_{K} $ by a Tikhonov regularization scheme involving the sample $ \bm{Z} $, that is
    \begin{equation*}
        \hat{\beta} = \arg\min_{\beta \in \mathcal{H}_{K}} \frac{1}{n} \left\{ \sum_{i=1}^{n} \left( \left< \beta, X_{i} \right>_{2} - Y_{i} \right)^2 + \lambda \| \beta \|_{K}^2 \right\}.
    \end{equation*}
    It has been well studied in the literature \cite{Cai2012,Yuan2010}.
    In contrast, we investigate the following online SGD algorithm in $\mathcal{H}_{K}$ generated by Tikhonov regularization scheme, which is also referred to as the \emph{online regularized learning algorithm}.
    It starts from $ \beta_1 = 0 $ and is then iteratively expressed as
    \begin{equation}
        \label{eq:SGDalgorithm}
        \beta_{k+1} = \beta_{k} - \gamma_{k} \left[ \left( \int_{\mathcal{T}} \beta_{k}(t) X_{k}(t) dt - Y_{k} \right) \int_{\mathcal{T}} K(s,\cdot) X_{k}(s) ds + \lambda \beta_{k} \right],
    \end{equation}
    for $ k=1,\cdots,n $ where $ \lambda = \lambda(n) > 0 $ is the regularization parameter and $ \gamma_{k} > 0 $ is the step-size.

    For any $ \ell \in \mathbb{N}_{+} $, we simply denote by $ [\ell] $ the set $ \{ 1, \cdots ,\ell\} $.
    The algorithm above generates the learning sequence $ \{ \beta_{k+1}: k \in [n]\} $ where each output $ \beta_{k+1} $ is a random variable depending only on $ \{ Z_{i} \}_{i=1}^{k} $.
    The estimation performance of $ \beta_{k+1} $ can be measured by the squared $ \mathcal{H}_{K} $ norm, i.e., $\left\| \beta_{k+1} - \beta^{*} \right\|_{K}^2 $.
    As a product of the proposed algorithm (\ref{eq:SGDalgorithm}), the functional sequence $ \{ \hat{\eta}_{k+1}=\left< \beta_{k+1},\cdot \right>: k \in [n]\} $ can be utilized to consider the prediction problem.
    For a functional $ \eta:\mathcal{L}^2 (\mathcal{T}) \to \mathbb{R} $, the prediction error can be characterized with squared loss
    $$ \mathcal{E}( \eta) := \mathbb{E}_{(X,Y)} [ Y - \eta(X) ]^2,$$
    where the expectation is taken with respect to the distribution of $(X, Y )$ in Model (\ref{eq:linearmodel}).
    Hence, the prediction performance of $ \hat{\eta}_{k+1} $ can be naturally measured by the excess prediction error
    \begin{equation*}
        \begin{aligned}
            \mathcal{E}( \hat{\eta}_{k+1}) - \mathcal{E}( \eta^{*}) := & \mathbb{E}_{(X,Y)} [ Y - \hat{\eta}_{k+1}(X) ]^2 - \mathbb{E}_{(X,Y)} [ Y - \eta^{*}(X) ]^2 \\
            = & \mathbb{E}_{X} [ \hat{\eta}_{k+1}(X) - \eta^{*}(X) ]^2.
        \end{aligned}
    \end{equation*}
    Based on the sample $ \bm{Z} $, we investigate two different types of step-sizes and show that how the type of the step-size controls the learning performance of the solution.
    The first one is a polynomially decaying sequence of the form $ \{ \gamma_{k} = \gamma_1 k^{-\mu} : k \in [n]\} $ with some $ \gamma_1 >0 $ and $ \mu \in (0,1) $ and the other one is a constant sequence of the form $ \{ \gamma_{k} = \gamma: k \in [n] \} $ with $ \gamma = \gamma(n) $ depending on the total number of iterations $ n $ (the sample size).
    Both types of step-sizes can play a role as an implicit regularization to ensure the generalization ability of algorithm (\ref{eq:SGDalgorithm}).
    Our main results show that explicit convergence rates of the excess prediction error and estimation error for algorithm (\ref{eq:SGDalgorithm}) can be obtained with respect to expectation norm under standard assumptions.
    More importantly, the saturation boundary appeared in the previous results \cite{Chen2022} can be improved by introducing an explicit regularization term in the convergence analysis of prediction error and estimation error with polynomially decaying step-size.
    To our best of knowledge, this is the first result establishing the capacity independent convergence analysis of estimation error for online learning algorithms with polynomially decaying step-size in the context of functional linear model (\ref{eq:linearmodel}).
    Moreover, we present more refined results leading to faster convergence rates under an additional capacity condition encoding the smoothness of kernels and covariance function.
    Our work extends the online regularized learning algorithm to functional linear model (\ref{eq:linearmodel}) and enriches the applications of RKHS methods to functional data analysis.
    We hope this work shall provide insight for better understanding the difference between the prediction problem and estimation problem in functional linear models.

    The rest of this paper is organized as follows.
    Our main results are detailed after some basic notations and assumptions in Section \ref{sec:results}.
    Section \ref{sec:relatedwork} collects discussion and comparisons with related work.
    In Section \ref{sec:decomposition}, we propose a novel decomposition for the excess prediction error and estimation error of algorithm (\ref{eq:SGDalgorithm}) and provide some preliminary results.
    The proofs of main results are given in Section \ref{sec:proofOfthm12} and \ref{sec:proofOfthm34}.
    Appendix presents proofs of some results in Section \ref{sec:decomposition}.

\section{Main Results}
\label{sec:results}
    In this section, we present our main results on the convergence rates of the excess prediction error and estimation error for algorithm (\ref{eq:SGDalgorithm}) with different step-sizes, respectively.

    We begin with some notations and assumptions.
    Let $ (\mathcal{H}, \| \cdot \|_{\mathcal{H}}) $ be a Hilbert space and $ A:\mathcal{H} \to \mathcal{H} $ be a linear operator.
    The operator norm of $ A $ is defined by $ \| A \|_{\mathcal{H}} = \sup_{\| f \|_{\mathcal{H}} = 1} \| Af \|_{\mathcal{H}} $.
    When $ \mathcal{H} $ is clear from context, we will omit the subscript and simply denote by $ \| \cdot \| $.
    For any bounded self-adjoint operators $ A $ and $ B $ on $ \mathcal{L}_{2}(\mathcal{T}) $, we write $ A \succeq B $ (or $ B \preceq A $) if $ A-B $ is positive semi-definite.
    For $ k \in \mathbb{N} $, we denote by $ \mathbb{E}_{Z^{k} } $ taking expectation with respect to $ \{ Z_1,\cdots ,Z_{k} \} $ and $ \mathbb{E}_{Z^{0} } [\xi] = \xi $ for any random variable $ \xi $. In addition, let $ a_{k} \lesssim b_{k} $ denote that there exists $ c > 0 $ such that $ a_{k} \leqslant c b_{k} $ for any $ k \geqslant 1 $.
    For the sake of simplicity, (\ref{eq:SGDalgorithm}) can be equivalently written as
    \begin{equation}
        \label{eq:SGDoperator}
        \beta_{k+1} = \beta_{k} - \gamma_{k} \left[ \left( \left< \beta_{k}, X_{k} \right>_{2} - Y_{k} \right) L_{K} X_{k} + \lambda \beta_{k} \right],
    \end{equation}
    where $ L_{K}: \mathcal{L}^2(\mathcal{T}) \to \mathcal{L}^2(\mathcal{T}) $ is an integral operator defined with the reproducing kernel $ K $, by
    \begin{equation}
        L_{K} f = \int_{\mathcal{T}} K(s,\cdot)f(s)ds, \qquad \forall f \in \mathcal{L}^2(\mathcal{T}).
    \end{equation}
    Recall that $ X $ is a square integrable random function over $ \mathcal{T} $ with zero mean value, then the covariance function of $ X $ is defined as
    \begin{equation*}
        C(s,t) = \mathbb{E} [ ( X(s) - \mathbb{E} [ X(s) ] ) ( X(t) - \mathbb{E} [ X(t) ] ) ] = \mathbb{E} [ X(s)X(t) ],
    \end{equation*}
    which can be easily verified to be real, symmetric and positive semi-definite.
    Throughout this paper, we assume $ C $ to be continuous and $ \mathbb{E} \| X \|_{2}^2 <\infty $.
    Hence, $ C $ is another Mercer kernel and we can define the integral operator $ L_{C} $ analogously.
    Since $ K $ and $ C $ are Mercer kernels on $ \mathcal{T} \times \mathcal{T} $ and $ \mathcal{T} $ is compact, these two functions are bounded which means that there exist two finite constants $ \kappa_1 $ and $ \kappa_{2} $ such that
    \begin{equation}
        \label{eq:boundforKandC}
        \kappa_1^2 = \max_{t \in \mathcal{T}} K(t,t) < \infty, \quad \text{and} \quad \kappa_2^2 = \max_{t \in \mathcal{T}} C(t,t) < \infty.
    \end{equation}
    In addition, $ L_{K} $ and $ L_{C} $ are compact, positive and thus of trace class.
    The corresponding $ r $-th power $ L_{K}^{r} $ and $ L_{C}^{r} $ are well-defined for any $ r > 0 $ according to the spectral theorem.
    The compactness of $ L_{K} $ implies that there exists an orthonormal system $ \{ \lambda_{\ell},\phi_{\ell} \}_{\ell\geqslant 1} $ in $ \mathcal{L}^2(\mathcal{T}) $ such that $ L_{K} = \sum_{\ell \geqslant 1} \lambda_{\ell} \phi_{\ell} \otimes \phi_{\ell} $ where the eigenvalues $ \{ \lambda_{\ell} \}_{\ell\geqslant 1} $(with geometric multiplicities) are non-negative and sorted in descending order.
    For notational simplicity, let
    \begin{equation}
        \label{eq:LscrMscr}
        T_{K} = L_{C}^{1 / 2 } L_{K} L_{C}^{1 / 2 } \quad \text{ and } \quad T_{C} = L_{K}^{1 / 2 } L_{C} L_{K}^{1 / 2 },
    \end{equation}
    which are naturally compact and positive.

    In the following, we introduce two key assumptions, which are commonly adopted in statistic learning and functional data analysis.
    \begin{assumption}
        \label{assum:bounded4thmoment}
        There exists a constant $ c > 0 $ such that for any $f \in \mathcal{L}^2 (\mathcal{T})$,
        \begin{equation}
            \label{eq:bounded4thmoment}
            \mathbb{E} \left<X,f \right>_{2}^{4} \leqslant c \left[ \mathbb{E} \left<X,f \right>_{2}^{2} \right]^2.
        \end{equation}
    \end{assumption}
    \noindent Condition (\ref{eq:bounded4thmoment}) indicates that all linear functionals of $ X $ have bounded kurtosis, which is adopted in \cite{Cai2012,Chen2022,guoCapacityDependentAnalysis2022,Yuan2010}.
    Especially, $ c=3 $ when $ X $ is Gaussian process.
    The second assumption is about regularizing condition of the target function $ \beta^{*} $.
    \begin{assumption}
        \label{assum:regularityBeta}
        \begin{equation}
            \label{eq:regularityBeta}
            L_{C}^{1 / 2 } \beta^{*} = T_{K}^{\theta} g^{*}, \quad \text{for some } g^{*} \in \mathcal{L}^2(\mathcal{T}) \text{ and } 0< \theta \leqslant 1.
        \end{equation}
    \end{assumption}
    Compared with the regularity assumption $ \beta^{*}\in \mathcal{H}_{K} $ \cite{Cai2012,Yuan2010,Fan2019}, Assumption \ref{assum:regularityBeta} is milder.
    This is well understood from Theorem 3 and Remarks in \cite{Chen2022} which suggests that if $ L_{K}^{\tau} \succeq \delta L_{C}^{v} $ for any $ \tau,\delta,v>0 $ and $ \tau + v \geqslant 1 $, Assumption \ref{assum:regularityBeta} is satisfied with $ \theta = 1 / (2 + 2v / \tau) $ for any $ \beta^{*} \in \mathcal{L}^2(\mathcal{T}) $.
    We can conclude that Assumption \ref{assum:regularityBeta} holds at least for any $ 0 < \theta < 1 /2 $ with proper selection of the reproducing kernel $ K $.
    In addition, Assumption \ref{assum:regularityBeta} is a weaker constraint on $ L_{K} $ and $ L_{C} $ while the FPCA-based methods require the perfect alignment between $ L_{K} $ and $ L_{C} $, that is, they share a common ordered set of eigenfunctions.

    Now we are in a position to present our main results. The first one establishes explicit rates of the excess prediction error in expectation with polynomially decaying step-size.
    \begin{theorem}
        \label{thm:1}
        Define $ \{ \hat{\eta}_{k+1}=\langle\beta_{k+1},\cdot\rangle: k \in [n]\} $ through (\ref{eq:SGDalgorithm}).
        Under Assumption \ref{assum:bounded4thmoment} and Assumption \ref{assum:regularityBeta} with $ 0<\theta \leqslant 1 $, set $ \gamma_{k} = \gamma_1 k^{-\mu} $ with $ \mu = \frac{2 \theta}{ 2 \theta + 1} $ and
        \begin{equation}
            \lambda =
            \begin{cases}
                n^{-\frac{1}{2 \theta + 1}}, & \text{if } 0 < \theta \leqslant \frac{1}{2}, \\
                n^{-\frac{1}{2 \theta + 1} + \frac{\varepsilon}{2 \theta}}, & \text{if } \frac{1}{2} < \theta \leqslant 1,
            \end{cases}
        \end{equation}
        for any $ 0 < \varepsilon < 2 \theta / (2 \theta + 1) $.
        If $ 0 < \gamma_1 \leqslant [ 4 C_{\mu} (1+c) ( 1 + \kappa_1^2 \kappa_2^2 )^2 ( \log 2 + \min \{ \mu, 1-\mu \}^{-1} ) ]^{-1} $ (where  constant $ C_{\mu} $ will be specified in the proof of Lemma \ref{lem:seriesgamma2}), then
        \begin{equation}
            \mathbb{E}_{Z^{n} }\left[\mathcal{E} ( \hat{\eta}_{n+1})\right]- \mathcal{E}( \eta^{*})  \lesssim
            \begin{cases}
                n^{-\frac{2 \theta}{2 \theta + 1}} \log(n+1), & 0 < \theta \leqslant \frac{1}{2}, \\
                n^{-\frac{2 \theta}{2 \theta + 1} + \varepsilon} , & \frac{1}{2} < \theta \leqslant 1.
            \end{cases}
        \end{equation}
    \end{theorem}

    We see from Theorem \ref{thm:1} that the convergence rate obtained with polynomially decaying step-size is of the form $ \mathcal{O}(n^{-\frac{2 \theta}{2 \theta + 1}} \log(n+1)) $ when $ 0<\theta \leqslant 1 / 2 $, which coincides with that of unregularized online learning algorithm (i.e., algorithm (\ref{eq:SGDalgorithm}) with $\lambda=0$) in \cite{Chen2022}.
    The results in \cite{Chen2022} are saturated at $\theta=1/2,$ that is, the derived rate ceases improving when $\theta>1/2$ and can not be faster than $\mathcal{O}(n^{-1/2} \log(n+1))$.
    By contrast, our convergence rates can be faster than $\mathcal{O}(n^{-1/2} \log(n+1))$ when $ 1/2<\theta\leqslant 1$, though they also suffer from the saturation: the convergence rate $ \mathcal{O}(n^{ - \frac{2}{3} + \varepsilon} ) $ will not be improved when $ \theta \geqslant 1 $.

    Faster convergence rates can be derived under some favorable conditions, in which the following capacity condition is commonly used in nonparametric function estimation.

    \begin{assumption}
        \label{assum:capacity}
        \[
            \operatorname{Tr}(T_{K}^{s} ) < \infty, \quad \text{for some } 0 < s \leqslant 1.
        \]
    \end{assumption}

    $ \operatorname{Tr}(\cdot) $ in Assumption \ref{assum:capacity} denotes the trace of a compact and positive semi-definite operator.
    Recall that $ L_{K} $, $ L_{C} $, $ T_{K} $ and $ T_{C} $ are of trace class.
    It is easy to verify that $ \operatorname{Tr}(T_{K}^{s} ) = \operatorname{Tr}(T_{C}^{s} ) $ and Assumption \ref{assum:capacity} holds trivially with $ s = 1 $.
    It should be pointed that Assumption \ref{assum:capacity} with $0<s<1$ implies that the eigenvalues $\left\{\sigma_i\right\}_{i=1}^\infty$ of $ T_{K} $ (sorted in decreasing order) satisfy a polynomial decay of order $1/s,$ i.e., $ \sigma_{i} \lesssim i^{-1 / s} $ for $ i \ge 1 $.
    A small $ s $ implies a fast polynomially decaying rate of $ \{ \sigma_{i} \} $ and further helps to bound the projection of target function on hypothesis space \cite{guoCapacityDependentAnalysis2022,Guo2019a}.
    The capacity of the hypothesis space $\mathcal{H}_K$ can also be measured by VC dimension, covering number, entropy number and the effective dimension.
    We refer the readers to these papers \cite{cucker2007,steinwart2008support,Bauer2007a,Caponnetto2007,Guo2017,Guo2019} and references therein.
    Next, we show faster rates of the excess prediction error under Assumption \ref{assum:capacity}.

    \begin{theorem}
        \label{thm:2}
        Define $ \{ \hat{\eta}_{k+1}=\langle\beta_{k+1},\cdot\rangle: k \in [n]\} $ through (\ref{eq:SGDalgorithm}).
        Under Assumption \ref{assum:bounded4thmoment}, Assumption \ref{assum:regularityBeta} with $ 0<\theta \leqslant 1 $ and Assumption \ref{assum:capacity} with $ 0 < s <1 $, set $ \gamma_{k} = \gamma_1 k^{-\mu} $ with $ \mu = \frac{\min\{2 \theta, 2-s \}}{ \min\{2 \theta, 2-s \} + 1} $ and $ \lambda = n^{-\frac{1}{\min\{2 \theta, 2-s \} + 1}}  $.
        If $ 0 < \gamma_1 \leqslant \min\left\{ (1+\kappa_1^2 \kappa_2^2)^{-1}, C_{1}^{\mathsf{S}} \right\} $ (where constant $ C_{1}^{\mathsf{S}} $ will be specified in Corollary \ref{coro:seriesConstantBoundcapa}), then
        \begin{equation}
            \mathbb{E}_{Z^{n} }\left[\mathcal{E} ( \hat{\eta}_{n+1}) \right]- \mathcal{E}( \eta^{*}) \lesssim n^{-\mu}=
            \begin{cases}
                n^{-\frac{2\theta}{2\theta+1}},& \text{if } 2\theta\leq 2-s,\\
                n^{-\frac{2-s}{3-s}},&\text{if }2\theta\geq 2-s.
            \end{cases}
        \end{equation}
        Moreover, when $ 2\theta > 2-s $, set $ \mu = \frac{2\theta}{2\theta+1}  $ and $ \lambda = n ^{-\frac{1}{2 \theta + 1} + \frac{\varepsilon}{2 \theta}} ,$ there also holds
        \begin{equation}
            \mathbb{E}_{Z^{n} }\left[\mathcal{E} ( \hat{\eta}_{n+1}) \right]- \mathcal{E}( \eta^{*}) \lesssim n^{-\frac{2\theta}{2\theta+1}  + \varepsilon},
         \end{equation}
         for any $ 0<\varepsilon< 2\theta / (2\theta+1) $.
    \end{theorem}

    The above bounds in Theorem \ref{thm:2} show that the logarithmic term of the convergence rates in Theorem \ref{thm:1} can be removed when $ \theta \leqslant 1 / 2 $, and the saturation boundary can be improved from $ \theta=1 / 2 $ to $\theta= 1/2 +(1-s)/2 $ under Assumption \ref{assum:capacity}  with $ s < 1 .$
    This coincides with that in \cite{guoCapacityDependentAnalysis2022}, where the convergence rates of prediction error for algorithm (\ref{eq:SGDalgorithm}) with $ \lambda = 0 $ saturate at $ \theta = (2-s) / 2 $ and can not be faster than $ \mathcal{O}(n^{- \frac{2-s}{3-s}} ) $.
    It is noteworthy that when $ \theta > (2-s) / 2 $, our convergence rates $\mathcal{O}(n^{-\frac{2\theta}{2\theta+1}  + \varepsilon})$ with any $ 0< \varepsilon< 2\theta / (2\theta+1)$ can be faster than $ \mathcal{O}(n^{- \frac{2-s}{3-s}} ) $.

    In Theorem \ref{thm:3} below, we present explicit convergence rates of the excess prediction error for algorithm (\ref{eq:SGDalgorithm}) with constant step-size.
    \begin{theorem}
        \label{thm:3}
        Define $ \{ \hat{\eta}_{k+1}=\langle\beta_{k+1},\cdot\rangle: k \in [n]\} $ through (\ref{eq:SGDalgorithm}).
        Under Assumption \ref{assum:bounded4thmoment}, Assumption \ref{assum:regularityBeta} with $ 0<\theta \leqslant 1 $ and Assumption \ref{assum:capacity} with $ 0 < s \leqslant 1 $, set the constant step-size $ \gamma_{k} = \gamma_{0} n^{-\mu} $ with $ \mu = \frac{2 \theta}{ 2 \theta + 1} $ and $ \lambda =  n^{-\frac{1}{2 \theta + 1}} $.
        If $ 0 < \gamma_0 \leqslant \min\{ C_{2}^{\mathsf{S}},C_{2*}^{\mathsf{S}} \} $ (where constants $ C_{2}^{\mathsf{S}},C_{2*}^{\mathsf{S}} $ will be specified in the proof), then
        \begin{equation}
            \mathbb{E}_{Z^{n} }\left[\mathcal{E} ( \hat{\eta}_{n+1}) \right] - \mathcal{E}( \eta^{*}) \leqslant
            \begin{cases}
                C^{\mathsf{fi}}_{\mathsf{p}} n^{-\frac{2 \theta}{2 \theta +1}} \log (n+1), & \text{if } s=1, \\
                C^{\mathsf{fi}}_{\mathsf{p,c}} n^{-\frac{2 \theta}{2 \theta +1}}, & \text{if } 0<s<1,
            \end{cases}
        \end{equation}
        where $ C^{\mathsf{fi}}_{\mathsf{p}} $ and $ C^{\mathsf{fi}}_{\mathsf{p,c}} $ are independent of $ n $ and  will be specified in the proof.
    \end{theorem}
    Compared with Theorem \ref{thm:1} for the capacity independent case (i.e., $s=1$ in Assumption \ref{assum:capacity}), Theorem \ref{thm:3} shows that the convergence rate $ \mathcal{O}(n^{-\frac{2 \theta}{2 \theta +1}} \log (n+1)) $ with constant step-size holds for any $ 0 < \theta \leqslant 1 $, which is also obtained in \cite{Chen2022,guoCapacityDependentAnalysis2022}.
    This indicates the advantage of constant step-size when dealing with finite samples, which can also be seen from the comparison between  the results of Theorem \ref{thm:2} and \ref{thm:3} in the capacity dependent setting.
    Finally, in contrast to the result corresponding to the capacity independent case, the logarithm factor can be removed under Assumption \ref{assum:capacity} with $ 0<s <1 $.

    In the following, we focus on the estimation error for $ \beta_{k+1} $ defined by (\ref{eq:SGDoperator}), i.e., $ \| \beta_{k+1} - \beta^{*} \|_{K} $.
    To this end, we first introduce another regularity assumption on $ \beta^{*} $ to replace Assumption \ref{assum:regularityBeta}.
    \begin{assumption}
        \label{assum:regularityBeta1}
        \begin{equation}
            \label{eq:regularityBeta1}
            \beta^{*} = L_{K}^{1 / 2 } T_{C}^{r} g_{*}, \quad \text{for some } g_{*} \in \mathcal{L}^2(\mathcal{T}) \text{ and } 0 < r \leqslant 1.
        \end{equation}
    \end{assumption}
    \noindent The assumption above implies that the slope function $ \beta^{*} \in \mathcal{H}_{K} $.
    In particular, if $ L_{C} \succeq \delta L_{K}^{v} $, Assumption \ref{assum:regularityBeta1} is guaranteed when $ \beta^{*} \in L_{K}^{1 /2+r(1+v)} (\mathcal{L}^2(\mathcal{T})) $ with $ 0 < r \leqslant 1 / (2+2v) $.
    This coincides with the general regularity assumption $ \beta^{*} \in L_{K}^{r} (\mathcal{L}^2(\mathcal{T})) $ with $ 0 < r \leqslant 1 $ applied in kernel-based learning algorithms \cite{Caponnetto2007,Guo2019a,Ying2008,Ying2006}.

    Theorem \ref{thm:4} below provides the convergence rates of estimation error for algorithm (\ref{eq:SGDalgorithm}) with polynomially decaying step-size.

    \begin{theorem}
        \label{thm:4}
        Define $ \{ \beta_{k+1}: k \in [n] \} $ by (\ref{eq:SGDalgorithm}).
        Under Assumption \ref{assum:bounded4thmoment}, Assumption \ref{assum:capacity} with $ 0 < s \leqslant 1 $ and Assumption \ref{assum:regularityBeta1} with $ 0 < r \leqslant 1 $, set $ \gamma_{k} = \gamma_1 k^{-\mu} $ with $ \mu = \frac{2 r + s}{ 2 r + s+ 1} $ and
        \begin{equation}
            \lambda =
            \begin{cases}
                n ^{-\frac{1}{2 r +s+ 1}}, & \text{if } 2r +s < 1, \\
                n ^{-\frac{1}{2 r +s+ 1} + \frac{\varepsilon}{2 r}}, & \text{if } 2r+s \geqslant 1,
            \end{cases}
        \end{equation}
        for any $ 0 < \varepsilon < 2 r / (2 r +s + 1) $.
        If $ 0 < \gamma_1 \leqslant [ 4 C_{\mu} (1+c) ( 1 + \kappa_1^2 \kappa_2^2 )^2 ( \log 2 + \min \{ \mu, 1-\mu \}^{-1} ) ]^{-1} $ (where the constant $ C_{\mu} $ will be specified in the proof of Lemma \ref{lem:seriesgamma2}), then
        \begin{equation}
            \mathbb{E}_{Z^{n} }\left[\left\| \beta_{n+1} - \beta^{*} \right\|_{K}^2\right] \lesssim
            \begin{cases}
                n^{-\frac{2 r}{2 r + s+ 1} } , & 2r +s < 1, \\
                n^{-\frac{2 r}{2 r + s+ 1} + \varepsilon}, & 2r +s \geqslant  1.
            \end{cases}
        \end{equation}
    \end{theorem}
    It is shown in Theorem \ref{thm:4} that we can obtain explicit convergence rate  $ \mathcal{O}(n^{-\frac{r}{r+1}+\varepsilon} ) $ for any $ 0 < \varepsilon < r / (r + 1) $ of estimation error for algorithm (\ref{eq:SGDalgorithm}) in the capacity independent case (i.e., Assumption \ref{assum:capacity} with $ s=1 $).
    To the best of our knowledge, the capacity independent convergence analysis remains an open problem for the unregularized online learning algorithm \cite{guoCapacityDependentAnalysis2022} with decaying step-size.
    In the capacity dependent setting, Theorem 3 in \cite{guoCapacityDependentAnalysis2022} shows that the rates of estimation error for algorithm (\ref{eq:SGDalgorithm}) with $\lambda=0$ saturate at $r=(1-s)/2$ and can not be faster than $ \mathcal{O}(n^{-\frac{1-s}{2}} \log(n+1) ) $, while our rate $ \mathcal{O} (n^{-\frac{2 r}{2 r + s+ 1} + \varepsilon}) $  can be faster than $ \mathcal{O}(n^{-\frac{1-s}{2}} \log(n+1) ) $ when $r>(1-s)/2$.

    Finally, we establish the convergence rate of estimation error with constant step-size.
    \begin{theorem}
        \label{thm:5}
        Define $ \{ \beta_{k+1}: k \in [n] \} $ by (\ref{eq:SGDalgorithm}).
        Under Assumption \ref{assum:bounded4thmoment}, Assumption \ref{assum:capacity} with $ 0 < s \leqslant 1 $ and Assumption \ref{assum:regularityBeta1} with $ 0 < r \leqslant 1 $, set the constant step-size $ \gamma_{k} = \gamma_{0} n^{-\mu} $ with $ \mu = \frac{2 r + s}{ 2 r + s + 1} $ and $ \lambda =  n^{-\frac{1}{2 r + s + 1} } $.
        If $ 0 < \gamma_0 \leqslant [ 2 (1+c) ( 1 + \kappa_1^2 \kappa_2^2 )^2 (1  +\mu^{-1})]^{-1} $, then
        \begin{equation}
            \mathbb{E}_{Z^{n} }\left\| \beta_{n+1} - \beta^{*} \right\|_{K}^2 \leqslant C^{\mathsf{fi}}_{\mathsf{e}} n^{-\frac{2 r}{2 r + s+ 1}},
        \end{equation}
        where the constant $ C^{\mathsf{fi}}_{\mathsf{e}} $ is independent of $ n $ and it will be specified in the proof.
    \end{theorem}
    The convergence rate in Theorem \ref{thm:5} holds for any $ 0 < r \leqslant 1 $.
    For a fixed $ r $, faster rate emerges with a stronger capacity assumption represented by a smaller $ s $.
    In addition, compared with Theorem \ref{thm:4}, Theorem \ref{thm:5} also demonstrates the same advantage of constant step-size when analyzing the estimation error.

\section{Related Work}
\label{sec:relatedwork}
    In this section, we discuss and compare our main results with related work.
    Recently, stochastic gradient descent methods have emerged as a fundamental and powerful tool for handling large-scale data sets, which allow processing data individually or with mini-batches to reduce computational complexity and memory requirement while preserving good convergence rates.
    There are many variants of SGD in an RKHS or general Hilbert spaces investigated in the literature \cite{Ying2008,Ying2006,Dieuleveut2016a, mucke2019beating,pillaud2018statistical,Lin2017a,guo2017online,Guo2019a,Guo2022b}.
    We refer the readers to these papers and references therein.
    In this paper, we enrich the theoretical analysis of functional linear regression by applying the regularized online learning algorithm (\ref{eq:SGDalgorithm}) in an RKHS and deriving explicit convergence rates of the excess prediction error and estimation error with appropriate parameters.
    To the best of our knowledge, this is the first result for online regularized algorithm in the context of functional linear regression in an RKHS.

    The kernel-based analysis of functional linear regression was first studied in \cite{Cai2012,Yuan2010}, where batch learning algorithms are proposed to estimate the slope function $ \beta^{*} $.
    Following the spirit of FPCA method, Yuan and Cai \cite{Yuan2010} derive the minimax optimal rate $ \mathcal{O}(n^{-2(s_1 + s_2) / (2(s_1 + s_2)+1)}) $ of the excess prediction error $ \mathcal{E} ( \hat{\eta}) - \mathcal{E}( \eta^{*}) $ under the assumptions that $ \beta^{*} \in \mathcal{H}_{K} $ and the decaying rates of eigenvalues $ \lambda_{i}(L_{K}) \asymp i^{-2 s_1} $, $ \lambda_{i}(L_{C}) \asymp i^{-2s_2} $ for $ s_1,s_2 > \frac{1}{2} $.
    Here $ a_{i} \asymp b_{i} $ for two positive sequence $ \{ a_{i} \} $ and $ \{ b_{i} \} $ means that $ a_{i} / b_{i} $ is bounded away from $ 0 $ and $ \infty $ as $ i \to \infty $.
    Afterwards, Cai and Yuan \cite{Cai2012} obtain the same result $ n^{-2(s_1+s_2) / (2(s_1+s_2)+1)} $ by relaxing the capacity dependent assumption with $ \lambda_{i}(T_{C}) \asymp i^{-2(s_1+s_2)} $.
    Compared with these works, we derive the convergence rates under more general assumption.
    As pointed in Section \ref{sec:results}, Assumption \ref{assum:regularityBeta} applied in bounding the excess prediction error presents a better characterization on the regularity of $ \beta^{*} $. For more discussion about the results between batch learning and online learning in the framework of functional linear regression, we refer the readers to \cite{guoCapacityDependentAnalysis2022}.

    As far as we know, the convergence analysis of SGD methods has not been studied in the context of functional linear regression in an RKHS till the very recent paper \cite{Chen2022} and \cite{guoCapacityDependentAnalysis2022} where the unregularized online learning algorithm (i.e., algorithm (\ref{eq:SGDalgorithm}) with $\lambda=0$) is investigated.
    Capacity independent analysis of the excess prediction error is established in \cite{Chen2022}, both prediction problems and estimation problems are studied in \cite{guoCapacityDependentAnalysis2022} via a novel capacity dependent analysis.

    For the excess prediction error, under Assumption \ref{assum:bounded4thmoment} and Assumption \ref{assum:regularityBeta} with $\theta>0$, if the step-size polynomially decays as $\gamma_k=\gamma_1 k^{-\min\left\{\frac12,\frac{2\theta}{2\theta+1}\right\}}$ for $k\in[n]$ with some proper constant $\gamma_1>0$, the convergence rate in \cite{Chen2022} is of the form $\mathcal{O}(n^{-\min\left\{\frac12,\frac{2\theta}{2\theta+1}\right\}}\log(n+1)) $.
    The rate is saturated at $\theta=1/2,$ which ceases improving when $\theta>1/2$ and can not be faster than $\mathcal{O}(n^{-1/2} \log(n+1)).$
    Under the same assumptions, our convergence rates $\mathcal{O}(n^{-\frac{2 \theta}{2 \theta + 1} + \varepsilon})$ with any $0<\varepsilon< 2\theta / ( 2\theta +1 ) $ in Theorem \ref{thm:1} can be faster than $\mathcal{O}(n^{-1/2} \log(n+1))$ when $ 1/2<\theta\leqslant 1$.
    Results in \cite{Chen2022} can be improved by capacity dependent analysis.
    Under Assumption \ref{assum:bounded4thmoment}, Assumption \ref{assum:regularityBeta} with $0<\theta\le1$, and the capacity assumption \ref{assum:capacity} with $0<s< 1$, our capacity dependent convergence rate $ \mathcal{O}(n^{-\frac{\min\{2\theta,2-s\}}{1+\min\{2\theta,2-s\}}}) $ in Theorem \ref{thm:2} matches that obtained in \cite{guoCapacityDependentAnalysis2022}, and is faster than that in \cite{Chen2022}.
    We can see the results also suffer from the saturation effect, and the best convergence rate $\mathcal{O}(n^{-\frac{2-s}{3-s}})$ will be achieved at $\theta=1-s/2.$
    Our results in Theorem \ref{thm:2} also show that the convergence rate can be of the form $\mathcal{O}( n^{-\frac{2 \theta}{2 \theta +1}+\varepsilon}) $ with any $0<\varepsilon<2\theta / ( 2\theta +1 )$ when $ \theta \geqslant 1-s/2$, which is better than $\mathcal{O}(n^{-\frac{2-s}{3-s}}).$
    For the case of constant step-size, under the same assumptions, our capacity independent convergence rate $ \mathcal{O}( n^{-\frac{2 \theta}{2 \theta +1}} \log (n+1))$ with $0<\theta\leqslant 1$ in Theorem \ref{thm:3} matches these in \cite{Chen2022} and \cite{guoCapacityDependentAnalysis2022}, of which the corresponding results do not suffer from the saturation effect when $\theta>1.$
    Moreover, the logarithm term $\log(n+1)$ can be removed in our results and \cite{guoCapacityDependentAnalysis2022} under Assumption \ref{assum:capacity} with $0<s< 1$.

    For the estimation error, it remains an open problem to conduct capacity independent analysis (i.e., $s=1$ in Assumption \ref{assum:capacity}) for the unregularized online learning algorithm (i.e., algorithm (\ref{eq:SGDalgorithm}) with $\lambda=0$) with polynomially decaying step-size.
    In light of this, Theorem \ref{thm:4} in this paper establishes convergence rates for estimation error of algorithm (\ref{eq:SGDalgorithm}) without capacity assumption, which can be further improved by the capacity assumption \ref{assum:capacity} with $0<s<1.$
    Under the same assumptions, our convergence rates match these in \cite{guoCapacityDependentAnalysis2022} when $0<r\leqslant 1$ for the case of constant step-size, while our result will suffer from the saturation effect when $r>1$ due to the regularization term added in our algorithm.
    When the step-size decays polynomially, we obtain the same convergence rates as these in \cite{guoCapacityDependentAnalysis2022} when $2r+s<1.$
    When $2r+s\ge1,$ the convergence rates in \cite{guoCapacityDependentAnalysis2022} saturate at $r=(1-s)/2,$ and it can not be faster than $ \mathcal{O}(n^{-\frac{1-s}{2}} \log(n+1) ) $.
    By contrast, one can easily see that our convergence rate $\mathcal{O}(n^{-\frac{2 r}{2 r + s+ 1} + \varepsilon})$ with arbitrary $0<\varepsilon<\frac{2 r}{2 r + s+ 1}$ in Theorem \ref{thm:5} can be faster than $ \mathcal{O}(n^{-\frac{1-s}{2}} \log(n+1) ) $ when $r\geqslant (1-s)/2$, which implies that the regularized online learning algorithm can uplift the saturation boundary.

    We conclude this section by some remarks.
    First, convergence rates of the excess prediction error and estimation error in expectation are established in this paper, and it would be interesting to convert them to probabilistic bounds such as those in \cite{Cai2012,Yuan2010}.
    Second, since the choice of regularization parameter $\lambda$ depends on the sample size $n,$ which requires knowing the sample size $n$ in advance, it would be interesting to study online regularized learning algorithm with time-dependent regularization parameters $ \{ \lambda_{k}\}_{k \in \mathbb{N}} $ \cite{Tarres2014} for functional linear regression.
    Third, our results demonstrate that the regularized online learning algorithms not only uplifts the saturating boundary for the excess prediction error, but also lead to a non-trivial convergence rates for estimation error without capacity assumption.
    Nevertheless, our results also suffer from the saturation effect when $\theta>1$ in Assumption \ref{assum:regularityBeta} and $r>1$ in Assumption \ref{assum:regularityBeta1}.
    It is unknown for us to get better error analysis to avoid this saturation effect.

\section{Error Decomposition and Preliminary Results}
\label{sec:decomposition}
    In this section, we propose a unified error decomposition to analyze the excess prediction error $ \mathcal{E}( \hat{\eta}_{k+1}) - \mathcal{E}( \eta^{*}) $ and estimation error $ \| \beta_{k+1}-\beta^{*} \|_{K} $ of online regularized algorithm (\ref{eq:SGDalgorithm}) for functional linear model (\ref{eq:linearmodel}).
    To this end, we first introduce some useful intermediate functions, of which one is referred to as \emph{regularizing function} and defined as
    \begin{equation}
        \label{eq:betalambda}
        \beta_{\lambda} = \arg \inf_{\beta \in \mathcal{H}_{K}} \left\{ \mathbb{E}_{(X,Y)} \left( \left< \beta,X \right>_{2} - Y \right)^{2} + \lambda \| \beta \|_{K}^2 \right\},
    \end{equation}
    for any $ \lambda >0 $.
    The following lemma provides an operator representation of $ \beta_{\lambda} $ which benefits to our error analysis.
    \begin{lemma}
        \label{lem:betalambda}
        Let $ \lambda >0 $ and $ \beta_{\lambda} $ be defined as (\ref{eq:betalambda}). Then we have that
        \begin{equation}
            \label{eq:betalambdaExpicit}
            L_{K} L_{C} (\beta_{\lambda} - \beta^{*}) + \lambda \beta_{\lambda} = 0 \text{  and  } L_{C}^{1 / 2 } \beta_{\lambda} = (\lambda I + T_{K})^{-1} T_{K} L_{C}^{1 / 2 } \beta^{*}.
        \end{equation}
    \end{lemma}
    \noindent In addition, another intermediate function $ f_{\lambda} $ is defined by
    \begin{equation}
        \label{eq:flambda}
        f_{\lambda} = (\lambda I + T_{C})^{-1} L_{K}^{1 / 2 } L_{C} \beta^{*}.
    \end{equation}
    By (\ref{eq:flambda}), we have that $ (\lambda I + L_{K} L_{C}) L_{K}^{1 / 2 } f_{\lambda} =  L_{K} L_{C} \beta^{*} $.
    Compared with $ \beta_{\lambda} $ in Lemma \ref{lem:betalambda}, it is obvious that $ L_{K}^{1 / 2 } f_{\lambda} $ is another solution to the equation $ (\lambda I + L_{K} L_{C}) \beta =  L_{K} L_{C} \beta^{*}. $

    With $ \beta_{\lambda} $ at hand, our analysis for the excess prediction error is carried out by the following decomposition
    \begin{equation}
        \beta_{k+1} - \beta^{*} = \beta_{k+1} - \beta_{\lambda} + \beta_{\lambda} - \beta^{*},
    \end{equation}
    of which the first term $ \beta_{k+1} - \beta_{\lambda} $ can be further decomposed as follows by the definitions of  $\beta_{k+1}$ and $\beta_{\lambda}$.
    \begin{lemma}
        \label{lem:induction}
        Let $ \lambda > 0 $ and $ \{ \beta_{k+1}: k \in [n] \} $ be defined as (\ref{eq:SGDalgorithm}).
        Then, for any $ k \in [n] $, we have that
        \begin{equation}
            \beta_{k+1} - \beta_{\lambda} = [I - \gamma_{k} ( L_{K}L_{C} + \lambda I) ] ( \beta_{k} - \beta_{\lambda} ) + \gamma_{k} \mathcal{B}_{k},
        \end{equation}
        where $ \mathcal{B}_{k} $ is a vector-valued random variable defined by
        \begin{equation}
            \label{eq:Bk}
            \mathcal{B}_{k} = L_{K}L_{C} (\beta_{k} - \beta^{*}) + \left( Y_{k} - \left< \beta_{k},X_{k} \right>_{2} \right)  L_{K} X_{k}.
        \end{equation}
    \end{lemma}
    \begin{proof}
        From the definition of $ \beta_{k+1} $ given by (\ref{eq:SGDoperator}), we have that
        \begin{equation*}
            \begin{aligned}
                \beta_{k+1} - \beta_{\lambda} & = \beta_{k} - \beta_{\lambda} + \gamma_{k} ( Y_{k} - \left< \beta_{k},X_{k} \right>_{2}) L_{K} X_{k} - \lambda \gamma_{k} \beta_{k} \\
                & = \left[ I - \gamma_{k} ( L_{K}L_{C} + \lambda I) \right]  (\beta_{k} - \beta_{\lambda}) + \gamma_{k}  L_{K}L_{C}  (\beta_{k} - \beta_{\lambda}) \\
                & \qquad - \lambda \gamma_{k} \beta_{\lambda} + \gamma_{k} ( Y_{k} - \left<  \beta_{k},X_{k} \right>_{2})  L_{K} X_{k}.
            \end{aligned}
        \end{equation*}
        Combining Lemma \ref{lem:betalambda} with the above equation yields the desired results.
    \end{proof}
    For $ \lambda > 0 $ and $ k \in \mathbb{N} $, define the operator $ \omega_{t}^{k} ( L_{K}L_{C} + \lambda I ) $ = $ \prod_{j=t}^{k} (  I - \gamma_{j} ( L_{K}L_{C} + \lambda I) ) $ for $ t \leqslant k, $ and $ \omega_{k+1}^{k} ( L_{K}L_{C} + \lambda I ) = I $.
    Then applying induction to Lemma \ref{lem:induction} yields that for any $ k \in [n] $,
    \begin{equation}
        \label{eq:Induction}
        \beta_{k+1} - \beta_{\lambda} = - \omega_{1}^{k} ( L_{K}L_{C} + \lambda I) \beta_{\lambda} + \sum_{i=1}^{k} \gamma_{i} \omega_{i+1}^{k} ( L_{K}L_{C} + \lambda I) \mathcal{B}_{i}.
    \end{equation}
    Analogously, we can derive the same result for $ \beta_{k+1} - L_{K}^{1 / 2 } f_{\lambda} $, i.e.,
    \begin{equation}
        \label{eq:KnormInduction}
        \begin{aligned}
            \beta_{k+1} - L_{K}^{1 / 2 } f_{\lambda} & = \left[ I - \gamma_{k} ( L_{K}L_{C} + \lambda I) \right] (\beta_{k} - L_{K}^{1 / 2 } f_{\lambda}) + \gamma_{k} \mathcal{B}_{k} \\
            &=- \omega_{1}^{k} ( L_{K}L_{C} + \lambda I) L_{K}^{1 / 2 } f_{\lambda} + \sum_{i=1}^{k} \gamma_{i} \omega_{i+1}^{k} ( L_{K}L_{C} + \lambda I) \mathcal{B}_{i}.
        \end{aligned}
    \end{equation}
    The following lemma is a classical application of the spectral theorem and provides essential estimates to our analysis.
    \begin{lemma}
        \label{lem:prodOperNormBound}
        Assume that $ \mathcal{A} $ is a compact positive operator on a real separable Hilbert space with $ \| \mathcal{A} \| \leqslant C_{*} $ for some $ C_{*}>0 $.
        Let $ \lambda > 0 $ and $ \gamma_{j}(C_{*}+\lambda) \leqslant 1 $ for any $ j \in [t,k] $.
        Then for any $ \alpha > 0 $,
        \begin{equation}
            \label{eq:prodOperNormBound}
            \left\| \mathcal{A}^{\alpha} \omega_{t}^{k} (\mathcal{A}+ \lambda I ) \right\|^2 \leqslant  \frac{ \left( \alpha / e \right)^{2 \alpha} + C_{*}^{2 \alpha} }{ \exp \left\{  \lambda \sum_{j=t}^{k} \gamma_{j} \right\} \left( 1 + \left(  \sum_{j=t}^{k} \gamma_{j} \right)^{ 2 \alpha}\right)}.
        \end{equation}
        Especially, the above estimates all hold true with $ \omega_{t}^{k} ( \mathcal{A} + \lambda I ) = I $ and $ \sum_{j=t}^{k} \gamma_{j} = 0 $ if $ t > k $.
    \end{lemma}

    In the following, we present the error decomposition of excess prediction error and estimation error for two different cases: the \emph{benign} case without Assumption \ref{assum:capacity} and the \emph{stringent} case under Assumption \ref{assum:capacity} with $0< s < 1 $.
    \begin{proposition}
        \label{prop:errorDECOM}
        Let $ \{ \hat{\eta}_{k+1}=\langle\beta_{k+1},\cdot\rangle: k \in [n]\} $ be defined by (\ref{eq:SGDalgorithm}).
        If $ \gamma_{k}(\kappa_1^2 \kappa_2^2+\lambda) \leqslant 1 $ for any $ k \in [n] $, then under Assumption \ref{assum:bounded4thmoment} we have
        \begin{equation}
            \label{eq:errorDECOM}
            \begin{aligned}
                \mathbb{E}_{Z^{k} }[\mathcal{E} ( \hat{\eta}_{k+1})] - \mathcal{E}( \eta^{*}) & \leqslant \left( \left\| \omega_{1}^{k} ( T_{K} + \lambda I) L_{C}^{1 / 2 } \beta_{\lambda} \right\|_{2} + \left\| L_{C}^{1 / 2 }  ( \beta_{\lambda} - \beta^{*}) \right\|_{2} \right)^2 \\
                & \quad + ( 1 + \kappa_1^2 \kappa_2^2 )^2 ( 1 + c ) \sum_{i=1}^{k}  \frac{ \gamma_{i}^2 \left( \sigma^2 + \mathbb{E}_{Z^{i-1} } [ \mathcal{E} (\hat{\eta}_{i}) ] - \mathcal{E}( \eta^{*}) \right) }{ \exp \left\{  \lambda \sum_{j=i+1}^{k} \gamma_{j} \right\} \left( 1 + \sum_{j=i+1}^{k} \gamma_{j}  \right)}.
            \end{aligned}
        \end{equation}
    \end{proposition}

    \begin{proof}
        Let $ \{ \beta_{k+1}: k \in [n]\} $ be defined by (\ref{eq:SGDalgorithm}), the  excess prediction error of $ \hat{\eta}_{k+1} $ can be rewritten as
        \begin{equation}
            \label{eq:DECOMExcessRisk}
            \begin{aligned}
                \mathcal{E} ( \hat{\eta}_{k+1}) - \mathcal{E}( \eta^{*}) = & \mathbb{E}_{X} [ \hat{\eta}_{k+1}(X) - \eta^{*}(X) ]^2  \\
                =& \mathbb{E}_{X} \left( \int_{\mathcal{T}}\left( \beta_{k+1}(s) - \beta^{*}(s) \right) \int_{\mathcal{T}} \left( \beta_{k+1}(t) - \beta^{*}(t) \right)  X(t) X(s) dt ds \right) \\
                =& \left< \beta_{k+1} - \beta^{*}, L_{C} (\beta_{k+1} - \beta^{*}) \right>_{2} = \left\| L_{C}^{1 / 2 }  (\beta_{k+1} - \beta^{*}) \right\|_{2}^2.
            \end{aligned}
        \end{equation}
        By the intermediate function $ \beta_{\lambda} $ in Lemma \ref{lem:betalambda}, $\left\| L_{C}^{1 / 2 }  (\beta_{k+1} - \beta^{*}) \right\|_{2}^2$ can be further decomposed as
        \begin{equation}
            \label{eq:predErrorDecom}
            \begin{aligned}
                &\left\| L_{C}^{1 / 2 }  (\beta_{k+1} - \beta^{*}) \right\|_{2}^2 = \left\| L_{C}^{1 / 2 }  (\beta_{k+1} - \beta_{\lambda} + \beta_{\lambda} - \beta^{*}) \right\|_{2}^2 \\
                =& \left\| L_{C}^{1 / 2 }  (\beta_{k+1} - \beta_{\lambda} ) \right\|_{2}^2 + \left\| L_{C}^{1 / 2 }  ( \beta_{\lambda} - \beta^{*}) \right\|_{2}^2 + 2 \left< L_{C}^{1 / 2 }  (\beta_{k+1} - \beta_{\lambda} ), L_{C}^{1 / 2 }  ( \beta_{\lambda} - \beta^{*}) \right>_{2}.
            \end{aligned}
        \end{equation}
        Note that for any $ k \in [n] $, $ L_{C}^{1 / 2 } ( I - \gamma_{k} L_{K} L_{C}) = L_{C}^{1 / 2 }  - \gamma_{k}L_{C}^{1 / 2 }  L_{K} L_{C}=( I - \gamma_{k}L_{C}^{1 / 2 } L_{K} L_{C}^{1 / 2 })L_{C}^{1 / 2 } $
        $= ( I - \gamma_{k} T_{K})  L_{C}^{1 / 2 }. $
        Then we can use (\ref{eq:Induction}) to represent the first term on the right-hand side of (\ref{eq:predErrorDecom}) as
        \begin{equation}
            \label{eq:residual1Decom}
            \begin{aligned}
                L_{C}^{1 / 2 }  (\beta_{k+1} - \beta_{\lambda}) = - \omega_{1}^{k} ( T_{K} + \lambda I) L_{C}^{1 / 2 } \beta_{\lambda} + \sum_{i=1}^{k} \gamma_{i} \omega_{i+1}^{k} ( T_{K} + \lambda I) L_{C}^{1 / 2 } \mathcal{B}_{i},
            \end{aligned}
        \end{equation}
        and it follows that
        \begin{equation}
            \label{eq:1stInErrorDecom}
            \begin{aligned}
                \mathbb{E}_{Z^{k} } \left\| L_{C}^{1 / 2 }  (\beta_{k+1} - \beta_{\lambda}) \right\|_{2}^2 &= \left\| \omega_{1}^{k} ( T_{K} + \lambda I) L_{C}^{1 / 2 } \beta_{\lambda} \right\|_{2}^2 + \mathbb{E}_{Z^{k} } \left\| \sum_{i=1}^{k} \gamma_{i} \omega_{i+1}^{k} ( T_{K} + \lambda I) L_{C}^{1 / 2 } \mathcal{B}_{i} \right\|_{2}^2 \\
                & \quad - 2 \mathbb{E}_{Z^{k} } \left< \omega_{1}^{k} ( T_{K} + \lambda I) L_{C}^{1 / 2 } \beta_{\lambda}, \sum_{i=1}^{k} \gamma_{i} \omega_{i+1}^{k} ( T_{K} + \lambda I) L_{C}^{1 / 2 } \mathcal{B}_{i} \right>_{2}\\
                &=:J_1+J_2+J_3.
            \end{aligned}
        \end{equation}
        In the following, we use Lemma \ref{lem:prodOperNormBound} to bound the three terms on the right-hand side of (\ref{eq:1stInErrorDecom}) denoted by $ J_1 $, $ J_2 $ and $ J_3 $, respectively.

        Since the sample $ \{ Z_{i} \}_{i=1}^{k} $ is $ i.i.d $ and $ \beta_{i} $ only depends on $ \{ Z_1,\cdots ,Z_{i-1} \} $ for any $ i \in [k] $, we use the definition of $ \mathcal{B}_{i} $ in (\ref{eq:Bk}) to have
        \begin{equation}
            \label{eq:0expectBk}
            \begin{aligned}
            \mathbb{E}_{Z_{i}} [ \mathcal{B}_{i} ] &=\mathbb{E}_{Z_{i}} \left[L_{K}L_{C} (\beta_{i} - \beta^{*}) + \left( Y_{i} - \left< \beta_{i},X_{i} \right>_{2} \right) L_{K} X_{i} \right]\\
            &=L_{K}L_{C} (\beta_{i} - \beta^{*}) + L_{K} \mathbb{E}_{X_{i}} \left[ \left< \beta^{*}-\beta_{i},X_{i} \right>_{2}  X_{i} \right]=0, \quad \forall i \in [k].
            \end{aligned}
        \end{equation}
        Note that the expectations are only taken for $ \mathcal{B}_{i}(i=1,\cdots ,k) $ in $ J_{3} $.
        Therefore, $ \mathbb{E}_{Z^{k} }[\mathcal{B}_{i}] = \mathbb{E}_{Z^{i-1} } \mathbb{E}_{Z_{i}} [ \mathcal{B}_{i} ] = 0 $ for any $ i \in [k] $ and thus we have $ J_{3}=0. $

        The second term $J_2$ in (\ref{eq:1stInErrorDecom}) can be rewritten as
        \[
            J_2 = \sum_{i=1}^{k} \sum_{i^\prime=1}^{k} \gamma_{i} \gamma_{i^\prime} \mathbb{E}_{Z^{k} } \left< \omega_{i+1}^{k} ( T_{K} + \lambda I) L_{C}^{1 / 2 } \mathcal{B}_{i}, \omega_{i^\prime+1}^{k} ( T_{K} + \lambda I) L_{C}^{1 / 2 } \mathcal{B}_{i^\prime} \right>_{2}.
        \]
        For any $ i > i^\prime $, there holds
        \[
            \begin{aligned}
                \mathbb{E}&_{Z^{i} } \left< \omega_{i+1}^{k} ( T_{K} + \lambda I) L_{C}^{1 / 2 } \mathcal{B}_{i}, \omega_{i^\prime+1}^{k} ( T_{K} + \lambda I) L_{C}^{1 / 2 } \mathcal{B}_{i^\prime} \right>_{2} \\
                = \mathbb{E}&_{Z^{i-1} } \mathbb{E}_{Z_{i}} \left< \mathcal{B}_{i}, L_{C}^{1 / 2 } \omega_{i+1}^{k} ( T_{K} + \lambda I) \omega_{i^\prime+1}^{k} ( T_{K} + \lambda I) L_{C}^{1 / 2 } \mathcal{B}_{i^\prime} \right>_{2} \\
                = \mathbb{E}&_{Z^{i-1} }  \left< \mathbb{E}_{Z_{i}} [\mathcal{B}_{i}], L_{C}^{1 / 2 } \omega_{i+1}^{k} ( T_{K} + \lambda I) \omega_{i^\prime+1}^{k} ( T_{K} + \lambda I) L_{C}^{1 / 2 } \mathcal{B}_{i^\prime}  \right>_{2} \\
                = 0\,&.
            \end{aligned}
        \]
        The same argument also holds for the case $ i < i^\prime $.
        Consequently, it follows that
        \begin{equation}
            \label{eq:J2}
            \begin{aligned}
                J_2 =& \sum_{i=1}^{k} \gamma_{i}^2  \mathbb{E}_{Z^{i} } \left\| \omega_{i+1}^{k} ( T_{K} + \lambda I) L_{C}^{1 / 2 } \mathcal{B}_{i} \right\|_{2}^2,
            \end{aligned}
        \end{equation}
        which can be bounded by the definition of $\mathcal{B}_i$ and (\ref{eq:0expectBk}) as
        \begin{equation}
            \label{eq:J2imediate}
            \begin{aligned}
                J_2 & \leqslant \sum_{i=1}^{k} \gamma_{i}^2 \mathbb{E}_{Z^{i} } \left[ ( Y_{i} - \left< \beta_{i},X_{i} \right>_{2})^2 \left\| \omega_{i+1}^{k} ( T_{K} + \lambda I) L_{C}^{1 / 2 } L_{K} X_{i} \right\|_{2}^2 \right]  \\
                &\leqslant \sum_{i=1}^{k} \gamma_{i}^2 \left\| \omega_{i+1}^{k} ( T_{K} + \lambda I) L_{C}^{1 / 2 } L_{K}^{1/2} \right\|_{2}^2 \mathbb{E}_{Z^{i} } \left[ ( Y_{i} - \left< \beta_{i},X_{i} \right>_{2})^2 \left\| L_{K}^{1/2} X_{i} \right\|_{2}^2 \right]  \\
                & = \sum_{i=1}^{k} \gamma_{i}^2 \left\| L_{C}^{1 / 2 } L_{K}^{1 / 2 } \omega_{i+1}^{k} ( T_{C} + \lambda I) \right\|^2   \mathbb{E}_{Z^{i} } \left[ ( Y_{i} - \left< \beta_{i},X_{i} \right>_{2})^2 \left\|  L_{K}^{1 / 2 }  X_{i} \right\|_{2}^2 \right] \\
                & = \sum_{i=1}^{k} \gamma_{i}^2 \left\| T_{C}^{1 / 2 } \omega_{i+1}^{k} ( T_{C} + \lambda I) \right\|^2 \mathbb{E}_{Z^{i-1} } \mathbb{E}_{X_{i}} \left[ \left( \left< \beta^{*} - \beta_{i} \right>_{2}^2 + \sigma^2 \right) \left\| L_{K}^{1 / 2 } X_{i} \right\|_{2}^2  \right],
            \end{aligned}
        \end{equation}
        where the last equality follows from $ \mathbb{E}_{\varepsilon_{i}} [( \left< \beta^{*} - \beta_{i}, X_{i} \right>_{2} + \varepsilon_{i} )^2] = \sigma^2 + \left< \beta^{*} - \beta_{i}, X_{i} \right>_{2}^2 $.
        Since $ L_{K} $ is compact and positive, we have $ \| L_{K}^{1 / 2 } X_{i} \|_{2}^2 = \sum_{\ell \geqslant 1} \lambda_{\ell} \left< \phi_{\ell}, X_{i} \right>_{2}^2 $.
        For $i\in[k],$ Assumption \ref{assum:bounded4thmoment} implies that
        \begin{equation}
            \label{eq:J2imediate2}
            \begin{aligned}
                \mathbb{E}_{X_{i} } \left[ \left< \beta^{*} - \beta_{i},X_{i} \right>_{2}^2 \left\| L_{K}^{1 / 2 } X_{i} \right\|_{2}^2 \right] &= \sum_{\ell \geqslant 1} \lambda_{\ell} \mathbb{E}_{X_{i} } \left[ \left< \phi_{\ell}, X_{i} \right>_{2}^2 \left< \beta^{*} - \beta_{i},X_{i} \right>_{2}^2 \right] \\
                & \leqslant \sum_{\ell \geqslant 1} \lambda_{\ell} \sqrt{\mathbb{E}_{X_{i} } \left< \phi_{\ell}, X_{i} \right>_{2}^4} \sqrt{\mathbb{E}_{X_{i} } \left< \beta^{*} - \beta_{i},X_{i} \right>_{2}^4} \\
                & \leqslant c \sum_{\ell \geqslant 1} \lambda_{\ell} \mathbb{E}_{X_{i} } \left[ \left< \phi_{\ell}, X_{i} \right>_{2}^2  \right]  \mathbb{E}_{X_{i} }  \left[ \left< \beta^{*} - \beta_{i},X_{i} \right>_{2}^2 \right] \\
                & = c \mathbb{E}_{X_{i} } \left[ \left\| L_{K}^{1 / 2 } X_{i} \right\|_{2}^2 \right]\left\| L_{C}^{1 / 2 } ( \beta^{*} - \beta_{i} ) \right\|_{2}^2.
            \end{aligned}
        \end{equation}
        Note that
        \begin{equation}
            \label{eq:consBoundLKXi}
            \begin{aligned}
                \mathbb{E}_{X_{i} } \left[ \left\| L_{K}^{1 / 2 } X_{i} \right\|_{2}^2 \right] & = \mathbb{E}_{X_{i} } \left< X_{i}, L_{K} X_{i} \right>_{2} = \int_{\mathcal{T}} \int_{\mathcal{T}} K(s,t) \mathbb{E}_{X_{i} } [X_{i}(s) X_{i}(t)] ds dt \\
                & = \int_{\mathcal{T}} \int_{\mathcal{T}} K(s,t) C(s,t) ds dt \leqslant \kappa_1^2 \kappa_2^2,
            \end{aligned}
        \end{equation}
        and $ \mathcal{E} (\hat{\eta}_{i}) - \mathcal{E}( \eta^{*}) = \| L_{C}^{1 / 2 } ( \beta^{*} - \beta_{i} ) \|_{2}^2 $.
        Since the boundedness of $ K $ and $ C $ in (\ref{eq:boundforKandC}) implies that $ \| L_{K} \|\leqslant \kappa_1^2 $ and $ \| L_{C} \| \leqslant \kappa_2^2 $, we have $ \| T_{C} \| \leqslant \kappa_1^2 \kappa_2^2 $.
        Applying Lemma \ref{lem:prodOperNormBound} with $ \alpha = 1 / 2  $, $ t = i +1 $ and $ k = k $ to $ T_{C} $ yields
        \begin{equation}
            \label{eq:omegaMNormBound}
            \left\| T_{C}^{1 / 2 } \omega_{i+1}^{k} ( T_{C} + \lambda I) \right\|^2 \leqslant \frac{  (2e)^{-1}  + \kappa_1^2 \kappa_2^2 }{ \exp \left\{  \lambda \sum_{j=i+1}^{k} \gamma_{j} \right\} \left( 1 + \sum_{j=i+1}^{k} \gamma_{j}  \right)}.
        \end{equation}
        Then combining the inequality above with (\ref{eq:J2imediate}), (\ref{eq:J2imediate2}) and (\ref{eq:consBoundLKXi}) implies that
        \begin{equation}
            \label{eq:J2Bound}
            \begin{aligned}
                J_2 & \leqslant (1+c) \sum_{i=1}^{k} \gamma_{i}^2 \kappa_1^2 \kappa_2^2   \frac{ \left( (2e)^{-1}  + \kappa_1^2 \kappa_2^2 \right) \left( \sigma^2 +  \mathbb{E}_{Z^{i-1} } [ \mathcal{E} (\hat{\eta}_{i}) ] - \mathcal{E}( \eta^{*}) \right) }{ \exp \left\{  \lambda \sum_{j=i+1}^{k} \gamma_{j} \right\} \left( 1 + \sum_{j=i+1}^{k} \gamma_{j}  \right)} \\
                & \leqslant  ( 1 + \kappa_1^2 \kappa_2^2 )^2 ( 1 + c ) \sum_{i=1}^{k}  \frac{ \gamma_{i}^2 \left( \sigma^2 + \mathbb{E}_{Z^{i-1} } [ \mathcal{E} (\hat{\eta}_{i}) ] - \mathcal{E}( \eta^{*}) \right)  }{ \exp \left\{  \lambda \sum_{j=i+1}^{k} \gamma_{j} \right\} \left( 1 + \sum_{j=i+1}^{k} \gamma_{j}  \right)}.
            \end{aligned}
        \end{equation}

        Next, we bound the last term on the right-hand side of (\ref{eq:predErrorDecom}).
        By (\ref{eq:0expectBk}) and (\ref{eq:residual1Decom}), we have
        \[
            \mathbb{E}_{Z^{k} } \left[ L_{C}^{1 / 2 }  (\beta_{k+1} - \beta_{\lambda}) \right] = - \omega_{1}^{k} ( T_{K} + \lambda I) L_{C}^{1 / 2 } \beta_{\lambda}.
        \]
        Therefore,
        \begin{equation}
            \label{eq:3rdInErrorDecom}
            \begin{aligned}
                \mathbb{E}_{Z^{k} } \left< L_{C}^{1 / 2 }  (\beta_{k+1} - \beta_{\lambda} ), L_{C}^{1 / 2 }  ( \beta_{\lambda} - \beta^{*}) \right>_{2} =& \left< \mathbb{E}_{Z^{k} } \left[  L_{C}^{1 / 2 }  (\beta_{k+1} - \beta_{\lambda} ) \right] , L_{C}^{1 / 2 }  ( \beta_{\lambda} - \beta^{*}) \right>_{2} \\
                \leqslant & \left\| \omega_{1}^{k} ( T_{K} + \lambda I) L_{C}^{1 / 2 } \beta_{\lambda} \right\|_{2} \left\| L_{C}^{1 / 2 }  ( \beta_{\lambda} - \beta^{*}) \right\|_{2}.
            \end{aligned}
        \end{equation}
        Finally, the proof is completed by combining (\ref{eq:1stInErrorDecom}), (\ref{eq:J2Bound}), (\ref{eq:3rdInErrorDecom}) and (\ref{eq:predErrorDecom}) together.
    \end{proof}

    \begin{proposition}
        \label{prop:errorDECOMcapa}
        Let $ \{ \hat{\eta}_{k+1}=\langle\beta_{k+1},\cdot\rangle: k \in [n]\} $ be defined by (\ref{eq:SGDalgorithm}).
        If $ \gamma_{k}(\kappa_1^2 \kappa_2^2+\lambda) \leqslant 1 $ for any $ k \in [n] $, then under Assumption \ref{assum:bounded4thmoment} and Assumption \ref{assum:capacity} with $ 0 < s < 1 $ we have
        \begin{equation}
            \label{eq:errorDECOMcapa}
            \begin{aligned}
                \mathbb{E}_{Z^{k} }[\mathcal{E} ( &\hat{\eta}_{k+1})]  - \mathcal{E}( \eta^{*}) \leqslant \left( \left\| \omega_{1}^{k} ( T_{K} + \lambda I) L_{C}^{1 / 2 } \beta_{\lambda} \right\|_{2} + \left\| L_{C}^{1 / 2 }  ( \beta_{\lambda} - \beta^{*}) \right\|_{2} \right)^2 +  \\
                & (\sqrt{c} + c) \operatorname{Tr} (T_{K}^{s} ) \sum_{i=1}^{k} \frac{ \gamma_{i}^2 \left( \sigma^2 +  \mathbb{E}_{Z^{i-1}} [\mathcal{E} (\hat{\eta}_{i})] - \mathcal{E}( \eta^{*}) \right) \left[ \left( \frac{2-s}{2\mathrm{e}} \right)^{2-s} + (\kappa_1 \kappa_2)^{4-2s} \right] }{ \exp \left\{  \lambda \sum_{j=i+1}^{k} \gamma_{j} \right\} \left( 1 + (\sum_{j=i+1}^{k} \gamma_{j}) ^{2-s} \right)}.
            \end{aligned}
        \end{equation}
    \end{proposition}
    \begin{proof}
        The difference between the proofs for Proposition \ref{prop:errorDECOM} and Proposition \ref{prop:errorDECOMcapa} is that the estimation of $ J_2 $ in (\ref{eq:J2}), which is defined by
        \[
            J_2 = \sum_{i=1}^{k} \gamma_{i}^2  \mathbb{E}_{Z^{i} } \left\| \omega_{i+1}^{k} ( T_{K} + \lambda I) L_{C}^{1 / 2 } \mathcal{B}_{i} \right\|_{2}^2.
        \]
        Combining the definition of $ \mathcal{B}_{i} $ and (\ref{eq:0expectBk}), we have
        \begin{equation}
            \begin{aligned}
                J_2 & \leqslant \sum_{i=1}^{k} \gamma_{i}^2 \mathbb{E}_{Z^{i} } \left[ ( Y_{i} - \left< \beta_{i},X_{i} \right>_{2})^2 \left\| \omega_{i+1}^{k} ( T_{K} + \lambda I) L_{C}^{1 / 2 } L_{K} X_{i} \right\|_{2}^2 \right]  \\
                & = \sum_{i=1}^{k} \gamma_{i}^2 \mathbb{E}_{Z^{i-1}} \mathbb{E}_{X_{i}} \left[ ( \sigma^2 + \left< \beta_{i} - \beta^{*},X_{i} \right>_{2}^2) \left\| \omega_{i+1}^{k} ( T_{K} + \lambda I) L_{C}^{1 / 2 } L_{K} X_{i} \right\|_{2}^2 \right] \\
                & = \sigma^2 \sum_{i=1}^{k} \gamma_{i}^2 \mathbb{E}_{X_{i}} \left\|  \omega_{i+1}^{k} ( T_{K} + \lambda I) L_{C}^{1 / 2 } L_{K} X_{i} \right\|_{2}^2 \\
                & \qquad + \sum_{i=1}^{k} \gamma_{i}^2 \mathbb{E}_{Z^{i-1}} \mathbb{E}_{X_{i}} \left[ \left< \beta_{i} - \beta^{*},X_{i} \right>_{2}^2 \left\| \omega_{i+1}^{k} ( T_{K} + \lambda I) L_{C}^{1 / 2 } L_{K} X_{i} \right\|_{2}^2 \right] \\
                & \leqslant \sigma^2 \sum_{i=1}^{k} \gamma_{i}^2 \left[ \mathbb{E}_{X_{i}} \left\|  \omega_{i+1}^{k} ( T_{K} + \lambda I) L_{C}^{1 / 2 } L_{K} X_{i} \right\|_{2}^4 \right]^{1/ 2} \\
                & \qquad + \sum_{i=1}^{k} \gamma_{i}^2 \mathbb{E}_{Z^{i-1}} \left[ \mathbb{E}_{X_{i}} \left< \beta_{i} - \beta^{*},X_{i} \right>_{2}^4  \right]^{1/2}  \left[ \mathbb{E}_{X_{i}} \left\| \omega_{i+1}^{k} ( T_{K} + \lambda I) L_{C}^{1 / 2 } L_{K} X_{i} \right\|_{2}^4 \right]^{1 /2},
            \end{aligned}
        \end{equation}
        where the first equality holds due to $ \mathbb{E}_{\varepsilon_{i}} [( \left< \beta^{*} - \beta_{i}, X_{i} \right>_{2} + \varepsilon_{i} )^2] = \sigma^2 + \left< \beta^{*} - \beta_{i}, X_{i} \right>_{2}^2 $ and we apply Cauchy-Schwarz inequality in the last inequality.
        By Assumption \ref{assum:bounded4thmoment} and Lemma 9 in \cite{guoCapacityDependentAnalysis2022}, we have
        \[
            J_2 \leqslant \sqrt{c} \sum_{i=1}^{k} \gamma_{i}^2 \left( \sigma^2+ \sqrt{c} \mathbb{E}_{Z^{i-1}} \mathbb{E}_{X_{i}}  \left< \beta_{i} - \beta^{*},X_{i} \right>_{2}^2 \right) \mathbb{E}_{X_{i}} \left\| \omega_{i+1}^{k} ( T_{K} + \lambda I) L_{C}^{1 / 2 } L_{K} X_{i} \right\|_{2}^2.
        \]
        Since $ \mathbb{E}_{X_{i}} \| A X_{i} \|_{2}^2 = \mathbb{E}_{X_{i}} \operatorname{Tr} [(A X_{i})\otimes (A X_{i})] = \mathbb{E}_{X_{i}} \operatorname{Tr} [A( X_{i}\otimes X_{i}) A^\prime] = \operatorname{Tr}(AL_{C}A^\prime) $ for any bounded linear operator $ A $ on $ \mathcal{L}^2 (\mathcal{T}) $ and $ \mathbb{E}_{X_{i}} \left< \beta_{i} - \beta^{*},X_{i} \right>_{2}^2 = \mathcal{E} (\hat{\eta}_{i}) - \mathcal{E}( \eta^{*}) $, applying $ A = \omega_{i+1}^{k} ( T_{K} + \lambda I) L_{C}^{1 / 2 } L_{K} $ and Assumption \ref{assum:capacity} yields
        \begin{equation}
            \begin{aligned}
                J_2 & \leqslant \sqrt{c} \sum_{i=1}^{k} \gamma_{i}^2 \left[ \sigma^2 + \sqrt{c} \left( \mathbb{E}_{Z^{i-1}} \mathcal{E} (\hat{\eta}_{i}) - \mathcal{E}( \eta^{*})\right) \right] \operatorname{Tr} \left[ T_{K}^2 \left( \omega_{i+1}^{k} ( T_{K} + \lambda I) \right)^2 \right] \\
                & \leqslant \sqrt{c} \sum_{i=1}^{k} \gamma_{i}^2 \left[ \sigma^2 + \sqrt{c} \left( \mathbb{E}_{Z^{i-1}} \mathcal{E} (\hat{\eta}_{i}) - \mathcal{E}( \eta^{*})\right) \right] \operatorname{Tr} (T_{K}^{s} ) \left\| T_{K}^{ 1 - \frac{s}{2}} \omega_{i+1}^{k} ( T_{K} + \lambda I) \right\|^2.
            \end{aligned}
        \end{equation}
        Finally, we complete the proof by applying Lemma \ref{lem:prodOperNormBound} to $ \left\| T_{K}^{ 1 - \frac{s}{2}} \omega_{i+1}^{k} ( T_{K} + \lambda I) \right\|^2 $ and following the proof of Proposition \ref{prop:errorDECOM}.
    \end{proof}
    We establish the error decomposition for estimation error as follows.
    \begin{proposition}
        \label{prop:errorDECOMKNorm1}
        Let $ \{ \beta_{k+1}: k \in [n]\} $ be defined by (\ref{eq:SGDoperator}) and $ \lambda > 0 $.
        If $ \gamma_{k}(\kappa_1^2 \kappa_2^2+\lambda) \leqslant 1 $ for any $ k \in [n] $, then under Assumption \ref{assum:bounded4thmoment} we have
        \begin{equation}
            \label{eq:errorDECOMKNorm1}
            \begin{aligned}
                \mathbb{E}_{Z^{k} } \left\| \beta_{k+1} - \beta^{*} \right\|_{K}^2 & \leqslant \left( \left\| \omega_{1}^{k} ( L_{K}L_{C} + \lambda I) L_{K}^{1 / 2 } f_{\lambda} \right\|_{K} + \left\| L_{K}^{1 / 2 } f_{\lambda} - \beta^{*} \right\|_{K} \right)^2 \\
                & \quad + \kappa_1^2 \kappa_2^2 (1+c)  \sum_{i=1}^{k}  \frac{ \gamma_{i}^2 \left( \sigma^2 + \mathbb{E}_{Z^{i-1} } [ \mathcal{E} (\hat{\eta}_{i}) ] - \mathcal{E}( \eta^{*}) \right)  }{ \exp \left\{  \lambda \sum_{j=i+1}^{k} \gamma_{j} \right\} }.
            \end{aligned}
        \end{equation}
    \end{proposition}
    \begin{proof}
        First, the estimation error can be decomposed as
        \begin{equation}
            \label{eq:KNormDecom}
            \left\| \beta_{k+1} - \beta^{*} \right\|_{K}^2 = \left\| \beta_{k+1} - L_{K}^{1 / 2 } f_{\lambda} \right\|_{K}^2 + \left\| L_{K}^{1 / 2 } f_{\lambda} - \beta^{*} \right\|_{K}^2 + 2 \left< \beta_{k+1} - L_{K}^{1 / 2 } f_{\lambda} , L_{K}^{1 / 2 } f_{\lambda} - \beta^{*} \right>_{K},
        \end{equation}
        where $f_{\lambda} $ is defined as (\ref{eq:flambda}).
        Since $ \mathbb{E}_{Z_{i}} [\mathcal{B}_{i}] = 0 $ for any $ i \in [k] $, by (\ref{eq:KnormInduction}) we have
        \begin{equation}
            \label{eq:1stInKNormErrorDecom}
            \begin{aligned}
                \mathbb{E}_{Z^{k}} \left\| \beta_{k+1} - L_{K}^{1 / 2 } f_{\lambda} \right\|_{K}^2 & = \left\| \omega_{1}^{k} ( L_{K}L_{C} + \lambda I) L_{K}^{1 / 2 } f_{\lambda} \right\|_{K}^2 + \mathbb{E}_{Z^{k}} \left\| \sum_{i=1}^{k} \gamma_{i} \omega_{i+1}^{k} ( L_{K}L_{C} + \lambda I) \mathcal{B}_{i} \right\|_{K}^2 \\
                & = \left\| \omega_{1}^{k} ( L_{K}L_{C} + \lambda I) L_{K}^{1 / 2 } f_{\lambda} \right\|_{K}^2 + \sum_{i=1}^{k} \gamma_{i}^2 \mathbb{E}_{Z^{i}} \left\| \omega_{i+1}^{k} ( L_{K}L_{C} + \lambda I) \mathcal{B}_{i} \right\|_{K}^2.
            \end{aligned}
        \end{equation}
        We can bound the second term on the right-hand side of (\ref{eq:1stInKNormErrorDecom}) in a similar spirit as (\ref{eq:J2imediate}), i.e.,
        \begin{equation}
            \label{eq:J2diate}
            \begin{aligned}
                & \sum_{i=1}^{k} \gamma_{i}^2 \mathbb{E}_{Z^{i}} \left\| \omega_{i+1}^{k} ( L_{K}L_{C} + \lambda I) \mathcal{B}_{i} \right\|_{K}^2 \\
                \leqslant & \sum_{i=1}^{k} \gamma_{i}^2 \mathbb{E}_{Z^{i} } \left[ ( Y_{i} - \left< \beta_{i},X_{i} \right>_{2})^2 \left\| \omega_{i+1}^{k} ( L_{K}L_{C} + \lambda I) L_{K} X_{i} \right\|_{K}^2 \right] \\
                \leqslant & \sum_{i=1}^{k} \gamma_{i}^2 \mathbb{E}_{Z^{i} } \left[ ( Y_{i} - \left< \beta_{i},X_{i} \right>_{2})^2 \left\| L_{K}^{1 / 2 } \omega_{i+1}^{k} ( T_{C} + \lambda I) L_{K}^{1 / 2 }  X_{i} \right\|_{K}^2 \right] \\
                \leqslant & \sum_{i=1}^{k} \gamma_{i}^2 \left\| \omega_{i+1}^{k} ( T_{C} + \lambda I) \right\|^2 \mathbb{E}_{Z^{i-1} }\mathbb{E}_{X^{i} } \left[ ( \left< \beta_{i} - \beta^{*},X_{i} \right>_{2}^2 + \sigma^2) \left\| L_{K}^{1 / 2 }  X_{i} \right\|_{2}^2 \right],
            \end{aligned}
        \end{equation}
        where the last inequality follows from that $ \| L_{K}^{1 / 2 } f \|_{K} = \| f \|_{2} $ for any $ f \in \mathcal{L}^2 ( \mathcal{T}) $.
        Further, combining (\ref{eq:J2imediate2}), (\ref{eq:consBoundLKXi}) and the fact $ \left\| \omega_{i+1}^{k} ( T_{C} + \lambda I) \right\| \leqslant \exp \left\{ - \lambda \sum_{j=i+1}^{k} \gamma_{j} \right\} $ implies that
        \begin{equation}
            \label{eq:boundJ2Prime}
            \begin{aligned}
                \sum_{i=1}^{k} \gamma_{i}^2 \mathbb{E}_{Z^{i}} \left\| \omega_{i+1}^{k} ( L_{K}L_{C} + \lambda I) \mathcal{B}_{i} \right\|_{K}^2 \leqslant \kappa_1^2 \kappa_2^2 (1+c)  \sum_{i=1}^{k}  \frac{ \gamma_{i}^2 \left( \sigma^2 + \mathbb{E}_{Z^{i-1} } [ \mathcal{E} (\hat{\eta}_{i}) ] - \mathcal{E}( \eta^{*}) \right)  }{ \exp \left\{  \lambda \sum_{j=i+1}^{k} \gamma_{j} \right\} }.
            \end{aligned}
        \end{equation}
        Note that $ \mathbb{E}_{Z^{k}} (\beta_{k+1} - L_{K}^{1 / 2 } f_{\lambda}) = - \omega_{1}^{k} ( L_{K}L_{C} + \lambda I) L_{K}^{1 / 2 } f_{\lambda} $. Therefore,
        \begin{equation}
            \label{eq:2ndInKNormErrorDecom}
            \mathbb{E}_{Z^{k}} \left< \beta_{k+1} - L_{K}^{1 / 2 } f_{\lambda} , L_{K}^{1 / 2 } f_{\lambda} - \beta^{*} \right>_{K} \leqslant \left\| \omega_{1}^{k} ( L_{K}L_{C} + \lambda I) L_{K}^{1 / 2 } f_{\lambda} \right\|_{K} \left\| L_{K}^{1 / 2 } f_{\lambda} - \beta^{*} \right\|_{K}.
        \end{equation}
        Finally, we complete the proof by substituting (\ref{eq:1stInKNormErrorDecom}), (\ref{eq:boundJ2Prime}) and (\ref{eq:2ndInKNormErrorDecom}) into (\ref{eq:KNormDecom}).
    \end{proof}

    \begin{proposition}
        \label{prop:errorDECOMKNormwithCapacity}
        Let $ \{ \beta_{k+1}: k \in [n]\} $ be defined by (\ref{eq:SGDoperator}) and $ \lambda > 0 $.
        If $ \gamma_{k}(\kappa_1^2 \kappa_2^2+\lambda) \leqslant 1 $ for any $ k \in [n] $, then under Assumption \ref{assum:bounded4thmoment} and Assumption \ref{assum:capacity} with $ 0 < s < 1 $ we have
        \begin{equation}
            \label{eq:errorDECOMKNormwithCapacity}
            \begin{aligned}
                \mathbb{E}_{Z^{k} }\left\| \beta_{k+1} - \beta^{*} \right\|&_{K}^2 \leqslant \left( \left\| \omega_{1}^{k} ( T_{K} + \lambda I) L_{C}^{1 / 2 } \beta_{\lambda} \right\|_{2} + \left\| L_{C}^{1 / 2 }  ( \beta_{\lambda} - \beta^{*}) \right\|_{2} \right)^2  + \sqrt{c} \operatorname{Tr} (T_{C}^{s} ) \cdot \\
                & \sum_{i=1}^{k}\frac{ \gamma_{i}^2 \left[ \sigma^2 + \sqrt{c} \left( \mathbb{E}_{Z^{i-1}} \mathcal{E} (\hat{\eta}_{i}) - \mathcal{E}( \eta^{*}) \right) \right] \left[ \left( \frac{1-s}{2\mathrm{e}} \right)^{1-s} + (\kappa_1 \kappa_2)^{2(1-s)}   \right]   }{ \exp \left\{  \lambda \sum_{j=i+1}^{k} \gamma_{j} \right\} \left( 1 + (\sum_{j=i+1}^{k} \gamma_{j}) ^{1-s} \right) }.
            \end{aligned}
        \end{equation}
    \end{proposition}
    \begin{proof}
        Based on the proof of Proposition \ref{prop:errorDECOMKNorm1}, we only need to analyze the second term on the right-hand side of (\ref{eq:1stInKNormErrorDecom}) under Assumption \ref{assum:capacity} with $0<s<1$.
        By (\ref{eq:J2diate}), we already have
        \[
            \sum_{i=1}^{k} \gamma_{i}^2 \mathbb{E}_{Z^{i}} \left\| \omega_{i+1}^{k} ( L_{K}L_{C} + \lambda I) \mathcal{B}_{i} \right\|_{K}^2 \leqslant \sum_{i=1}^{k} \gamma_{i}^2 \mathbb{E}_{Z^{i} } \left[ ( Y_{i} - \left< \beta_{i},X_{i} \right>_{2})^2 \left\|  \omega_{i+1}^{k} ( T_{C} + \lambda I) L_{K}^{1 / 2 }  X_{i} \right\|_{2}^2 \right].
        \]
        It follows that
        \[
            \begin{aligned}
                &\sum_{i=1}^{k} \gamma_{i}^2 \mathbb{E}_{Z^{i} } \left[ ( Y_{i} - \left< \beta_{i},X_{i} \right>_{2})^2 \left\|  \omega_{i+1}^{k} ( T_{C} + \lambda I) L_{K}^{1 / 2 }  X_{i} \right\|_{2}^2 \right]\\
                = & \sum_{i=1}^{k} \gamma_{i}^2 \mathbb{E}_{Z^{i-1}} \mathbb{E}_{X_{i}} \left[ ( \sigma^2 + \left< \beta_{i} - \beta^{*},X_{i} \right>_{2}^2) \left\|  \omega_{i+1}^{k} ( T_{C} + \lambda I) L_{K}^{1 / 2 } X_{i} \right\|_{2}^2 \right] \\
                \leqslant &\sum_{i=1}^{k} \gamma_{i}^2 \left( \sigma^2 + \left( \mathbb{E}_{Z^{i-1}} \mathbb{E}_{X_{i}} \left< \beta_{i} - \beta^{*},X_{i} \right>_{2}^4 \right)^{1 / 2 } \right) \left(  \mathbb{E}_{X_{i}} \left\|  \omega_{i+1}^{k} ( T_{C} + \lambda I) L_{K}^{1 / 2 } X_{i} \right\|_{2}^4 \right)^{1 / 2 } \\
                \leqslant & \sqrt{c} \sum_{i=1}^{k} \gamma_{i}^2  \left( \sigma^2 + \sqrt{c}  \mathbb{E}_{Z^{i-1}} \mathbb{E}_{X_{i}} \left< \beta_{i} - \beta^{*},X_{i} \right>_{2}^2 \right)  \mathbb{E}_{X_{i}} \left\|  \omega_{i+1}^{k} ( T_{C} + \lambda I) L_{K}^{1 / 2 } X_{i} \right\|_{2}^2
            \end{aligned}
        \]
        where the first inequality follows from Cauchy-Schwarz inequality, and we use Assumption \ref{assum:bounded4thmoment} and Lemma 9 in \cite{guoCapacityDependentAnalysis2022} in the last inequality.
        Since $ \mathbb{E}_{X_{i}} \left< \beta_{i} - \beta^{*},X_{i} \right>_{2}^2 = \mathcal{E} (\hat{\eta}_{i}) - \mathcal{E}( \eta^{*}) $, applying $ \mathbb{E}_{X_{i}} \| A X_{i} \|_{2}^2 = \mathbb{E}_{X_{i}} \operatorname{Tr} [(A X_{i})\otimes (A X_{i})] = \operatorname{Tr}(AL_{C}A^\prime) $ with $A=\omega_{i+1}^{k} ( L_{K}L_{C} + \lambda I)L_K^{1/2}$ and Assumption \ref{assum:capacity} with $ 0 < s < 1 $ yields
        \begin{equation}
            \begin{aligned}
                &\sum_{i=1}^{k} \gamma_{i}^2 \mathbb{E}_{Z^{i}} \left\| \omega_{i+1}^{k} ( L_{K}L_{C} + \lambda I) \mathcal{B}_{i} \right\|_{K}^2 \\
                \leqslant & \sqrt{c} \sum_{i=1}^{k} \gamma_{i}^2 \left[ \sigma^2 + \sqrt{c} \left(  \mathbb{E}_{Z^{i-1}} \mathcal{E} (\hat{\eta}_{i}) - \mathcal{E}( \eta^{*}) \right)  \right] \operatorname{Tr} \left[ T_{C} \left( \omega_{i+1}^{k} ( T_{C} + \lambda I) \right)^2 \right] \\
                \leqslant & \sqrt{c}  \operatorname{Tr} (T_{C}^{s} ) \sum_{i=1}^{k} \gamma_{i}^2 \left[ \sigma^2 + \sqrt{c} \left( \mathbb{E}_{Z^{i-1}} \mathcal{E} (\hat{\eta}_{i}) - \mathcal{E}( \eta^{*}) \right) \right] \left\| T_{C}^{\frac{1-s}{2}} \omega_{i+1}^{k} ( T_{C} + \lambda I) \right\|^2.
            \end{aligned}
        \end{equation}
        Finally, we complete the proof by applying Lemma \ref{lem:prodOperNormBound} to $ \left\| T_{C}^{\frac{1-s}{2}} \omega_{i+1}^{k} ( T_{C} + \lambda I) \right\|^2 $ and combining (\ref{eq:1stInKNormErrorDecom}), (\ref{eq:2ndInKNormErrorDecom}) and (\ref{eq:KNormDecom}) together.
    \end{proof}

    It is worthy emphasizing that the results in the above four propositions have a similar structure.
    We illustrate this to give a detailed analysis for the excess prediction error and estimation error.
    The first two terms referred to as the \emph{approximation error} are independent of samples while the third term can be treated as the \emph{cumulative sample error} \cite{Ying2008}.
    Before estimating these two types of errors in the following sections, we introduce some useful preliminary results to analyze the cumulative sample error in the above propositions.
    Note that it can be bounded by (omitting constants)
    \[
        \sum_{i=1}^{k}  \frac{ \gamma_{i}^2 \exp \left\{ - \lambda \sum_{j=i+1}^{k} \gamma_{j} \right\} }{  1 + \left( \sum_{j=i+1}^{k} \gamma_{j} \right)^{\nu}   } \max_{1\leqslant j\leqslant k} \left( \sigma^2 + \mathbb{E}_{Z^{i-1} } [ \mathcal{E} (\hat{\eta}_{i}) ] - \mathcal{E}( \eta^{*}) \right),
    \]
    with $ \nu \geqslant 0 $.
    Now we turn to estimate the uniform bounds for the excess prediction error and the series involving step-sizes, respectively.
    First, we establish the uniform bounds for the excess prediction error.
    \begin{proposition}
        \label{prop:UniformBounds}
        Let $ \{ \hat{\eta}_{k+1}=\langle\beta_{k+1},\cdot\rangle: k \in [n]\} $ be defined through (\ref{eq:SGDalgorithm}).
        If the step-sizes satisfy that $ \gamma_{k}(\kappa_1^2 \kappa_2^2+\lambda) \leqslant 1 $ and
        \[
            \sum_{i=1}^{k}  \frac{ \gamma_{i}^2 \exp \left\{ - \lambda \sum_{j=i+1}^{k} \gamma_{j} \right\} }{  1 + \sum_{j=i+1}^{k} \gamma_{j}} \leqslant \frac{1}{ 2 ( 1 + \kappa_1^2 \kappa_2^2 )^2 ( 1 + c ) },
        \]
        for any $ k \in [n] $, then under Assumption \ref{assum:bounded4thmoment} we have
        \begin{equation}
            \label{eq:UniformBounds}
            \mathbb{E}_{Z^{k-1} }[\mathcal{E} ( \hat{\eta}_{k})] - \mathcal{E}( \eta^{*}) \leqslant 8 \kappa_{2}^2 \left\| \beta^{*} \right\|_{2}^2 + \sigma^2, \qquad \text{for any } k \in [n+1].
        \end{equation}
    \end{proposition}
    The proof of Proposition \ref{prop:UniformBounds} will be given in the Appendix.
    Analogously, we can obtain the following uniform bounds with Proposition \ref{prop:errorDECOMcapa}.
    \begin{proposition}
        \label{prop:UniformBoundscapa}
        Let $ \{ \hat{\eta}_{k+1}=\langle\beta_{k+1},\cdot\rangle: k \in [n]\} $ be defined through (\ref{eq:SGDalgorithm}).
        If $ \gamma_{k}(\kappa_1^2 \kappa_2^2+\lambda) \leqslant 1 $ and
        \[
            \sum_{i=1}^{k}  \frac{ \gamma_{i}^2 \exp \left\{ - \lambda \sum_{j=i+1}^{k} \gamma_{j} \right\} }{  1 + \left(  \sum_{j=i+1}^{k} \gamma_{j} \right)^{2-s}  } \leqslant \frac{1}{ 2 (\sqrt{c} + c) \operatorname{Tr} (T_{K}^{s} ) \left[ \left( \frac{2-s}{2\mathrm{e}} \right)^{2-s} + (\kappa_1 \kappa_2)^{4-2s} \right] },
        \]
        for any $ k \in [n] $, then under Assumption \ref{assum:bounded4thmoment} and Assumption \ref{assum:capacity} with $ 0 < s < 1 $ we have
        \begin{equation}
            \label{eq:UniformBoundscapa}
            \mathbb{E}_{Z^{k-1} }[\mathcal{E} ( \hat{\eta}_{k})] - \mathcal{E}( \eta^{*}) \leqslant 8 \kappa_{2}^2 \left\| \beta^{*} \right\|_{2}^2 + \sigma^2, \qquad \text{for any } k \in [n+1].
        \end{equation}
    \end{proposition}
    We omit the proof of Proposition \ref{prop:UniformBoundscapa} since it is similar to that of Proposition \ref{prop:UniformBounds}.
    The following lemma is introduced to bound the series involved in our analysis.
    \begin{lemma}
        \label{lem:seriesgamma2}
        Let $ \lambda > 0 $ and $ \nu \geqslant 0 $.
        If $ \gamma_{i} = \gamma_1 i^{-\mu} $, $ i \in [k] $ with $ 0 < \mu,\gamma_1$ $ <1 $, then for any $ k \geqslant 1 $,
        \begin{equation}
            \label{eq:seriesBound}
            \begin{aligned}
                & \sum_{i=1}^{k} \frac{ \gamma_{i}^2 \exp \left\{ - \lambda \sum_{j=i+1}^{k} \gamma_{j} \right\}}{  1 + \left( \sum_{j=i+1}^{k} \gamma_{j} \right)^{\nu} }  \\
                \leqslant &
                \begin{cases}
                    C_{\mu}\gamma_1 \left(  \exp \left\{ - \lambda \gamma_1 d_{\mu} k^{1-\mu} \right\} k^{- \min \{ \mu ,1-\mu \} } + k^{-\mu}  \right)\log(k+1), & \nu = 1, 0<\mu<1, \\
                    \hat{C}_{\mu}\gamma_1 (k+1)^{-\mu + (1-\mu)(1-\nu)},  & 0 < \nu < 1, 0< \mu<\frac{1}{2}, \\
                    \hat{C}_{\mu} \gamma_1 [ (k+1)^{-\mu + (1-\mu)(1-\nu)} +\exp \left\{ - \lambda \gamma_1 d_{\mu} k^{1-\mu} \right\}k^{-\nu(1-\mu) } \log(k+1) ],& 0 < \nu < 1, \frac{1}{2}\leqslant  \mu<1, \\
                    \widetilde{C}_{\mu}\gamma_1 \left( \exp \left\{ - \lambda \gamma_1 d_{\mu} k^{1-\mu} \right\} + k^{1-2\mu} \right), & \nu = 0, 1 / 2  < \mu < 1, \\
                    \overline{C}_{\mu} \gamma_1 \left( \exp \left\{ - \lambda \gamma_1 d_{\mu} k^{1-\mu} \right\} k^{-\min \{ \mu, \nu(1-\mu) \}} + k^{-\mu}  \right) , & \nu >1, 0 < \mu < 1,
                \end{cases}
            \end{aligned}
        \end{equation}
        where $ d_{\mu} = ( 1-2^{\mu-1}  ) / ( 1-\mu )$, $ C_{\mu} $, $ \hat{C}_{\mu} $, $ \widetilde{C}_{\mu} $ and $ \overline{C}_{\mu} $ are independent of $ k $ and will be specified in the proof.
    \end{lemma}

    Moreover, the corollaries below collect some direct results of Lemma \ref{lem:seriesgamma2} to ensure that Proposition \ref{prop:UniformBounds} and \ref{prop:UniformBoundscapa} hold with a suitable choice of $ \gamma_1 $.
    \begin{corollary}
        \label{coro:seriesConstantBound}
        Let $ \lambda > 0 $ and $ 0 < \mu <1 $. If the step-sizes satisfy $ \gamma_{i} = \gamma_1 i^{-\mu} $, $ i \in [k] $ with $ \gamma_1 \leqslant \left[ 4 C_{\mu} (1+c) ( 1 + \kappa_1^2 \kappa_2^2 )^2 ( \log 2 + \min \{ \mu, 1-\mu \}^{-1} ) \right]^{-1} $, then for any $ k \geqslant 1 $,
        \begin{equation}
            \label{eq:seriesConstantBound}
            \begin{aligned}
                \sum_{i=1}^{k}  \frac{ \gamma_{i}^2 \exp \left\{ - \lambda \sum_{j=i+1}^{k} \gamma_{j} \right\}}{  1 + \sum_{j=i+1}^{k} \gamma_{j} } \leqslant \frac{1}{ 2 ( 1 + \kappa_1^2 \kappa_2^2 )^2 ( 1 + c ) }.
            \end{aligned}
        \end{equation}
    \end{corollary}
    \begin{proof}
        We see  from (\ref{eq:seriesBound}) that for any $ k \geqslant 1 $ and $ \lambda >0 $,
        \[
            \begin{aligned}
                \sum_{i=1}^{k}  \frac{ \gamma_{i}^2 \exp \left\{ - \lambda \sum_{j=i+1}^{k} \gamma_{j} \right\}}{  1 + \sum_{j=i+1}^{k} \gamma_{j} } & \leqslant 2 C_{\mu} \gamma_1 k^{-\min \{ \mu, 1-\mu \} } \log ( k+1 ) \\
                & \leqslant 2 C_{\mu} \gamma_1 \left[ \log 2 + k^{-\min \{ \mu, 1-\mu \} } \log  k  \right].
            \end{aligned}
        \]
        A simple calculation shows that
        \[
            \log k = \frac{1}{ \min \{ \mu,1-\mu \} } \log  k^{ \min \{ \mu,1-\mu \}  } \leqslant \frac{1}{ \min \{ \mu,1-\mu \} } k^{ \min \{ \mu,1-\mu \}  }.
        \]
        Combining the above two inequalities with the assumption on $ \gamma_1 $ yields (\ref{eq:seriesConstantBound}).
    \end{proof}

    \begin{corollary}
        \label{coro:seriesConstantBoundcapa}
        Let $ \lambda > 0 $ and $ 0 < \mu <1 $. If the step-sizes satisfy $ \gamma_{i} = \gamma_1 i^{-\mu} $, $ i \in [k] $ with $ \gamma_1 \leqslant C_{1}^{\mathsf{S}} = \left\{ 4 (\sqrt{c} + c) \overline{C}_{\mu} \operatorname{Tr} (T_{K}^{s} ) \left[ \left( \frac{2-s}{2\mathrm{e}} \right)^{2-s} + (\kappa_1 \kappa_2)^{4-2s} \right] \right\}^{-1} $ where $ 0 < s < 1 $, then for any $ k \geqslant 1 $,
        \begin{equation}
            \label{eq:seriesConstantBoundcapa}
            \begin{aligned}
                \sum_{i=1}^{k} \frac{ \gamma_{i}^2 \exp \left\{ - \lambda \sum_{j=i+1}^{k} \gamma_{j} \right\} }{  1 + \left(  \sum_{j=i+1}^{k} \gamma_{j} \right)^{2-s}  } \leqslant \frac{1}{ 2 (\sqrt{c} + c) \operatorname{Tr} (T_{K}^{s} ) \left[ \left( \frac{2-s}{2\mathrm{e}} \right)^{2-s} + (\kappa_1 \kappa_2)^{4-2s} \right] }.
            \end{aligned}
        \end{equation}
    \end{corollary}

    \begin{proof}
        Applying Lemma \ref{lem:seriesgamma2} with $ \nu = 2-s >1 $, we see that for any $ k \geqslant 1 $ and $ \lambda > 0 $,
        \[
            \begin{aligned}
                \sum_{i=1}^{k} \frac{ \gamma_{i}^2 \exp \left\{ - \lambda \sum_{j=i+1}^{k} \gamma_{j} \right\} }{  1 + \left(  \sum_{j=i+1}^{k} \gamma_{j} \right)^{2-s} } & \leqslant 2 \overline{C}_{\mu} \gamma_1.
            \end{aligned}
        \]
        Therefore, the proof is completed by the above inequality and the assumption on $ \gamma_1 $.
    \end{proof}

\section{Analysis for the Excess Prediction Error}
\label{sec:proofOfthm12}
    In this section, we analyze the generalization performance of the functional estimator $ \hat{\eta}_{k+1} =\langle\beta_{k+1},\cdot\rangle_2$ with $\beta_{k+1}$ defined in (\ref{eq:SGDalgorithm}) and present explicit convergence rates of the excess prediction error in expectation for algorithm (\ref{eq:SGDalgorithm}).
    For this purpose, we first establish basic estimation for the deterministic approximation errors including $ \| \omega_{1}^{k} ( T_{K} + \lambda I) L_{C}^{1 / 2 } \beta_{\lambda} \|_{2} $ and $ \| L_{C}^{1 / 2 }  ( \beta_{\lambda} - \beta^{*}) \|_{2} $ in Proposition \ref{prop:errorDECOM} and \ref{prop:errorDECOMcapa}.
    Then, we complete the proofs of Theorem \ref{thm:1}, \ref{thm:2} and \ref{thm:3} by combining the estimates for approximation error and cumulative sample error together and substituting the specific types of step-sizes into the derived bounds.

    \begin{proposition}
        \label{prop:errorDECOM2}
        Let $ \{ \hat{\eta}_{k+1}=\langle\beta_{k+1},\cdot\rangle: k \in [n]\} $ be defined through (\ref{eq:SGDalgorithm}).
        If $ \gamma_{k}(\kappa_1^2 \kappa_2^2+\lambda) \leqslant 1 $ for any $ k \in [n] $, then under Assumption \ref{assum:bounded4thmoment} and Assumption \ref{assum:regularityBeta} with $0<\theta\leqslant 1,$ we have
        \begin{equation}
            \label{eq:errorDECOM2}
            \begin{aligned}
                \mathbb{E}_{Z^{k} }[\mathcal{E}& ( \hat{\eta}_{k+1})] - \mathcal{E}( \eta^{*})  \leqslant  2 \left\| g^{*} \right\|_{2}^2 \frac{ \left( \theta / e \right)^{2 \theta} + ( \kappa_1 \kappa_2)^{4 \theta} }{ \exp \left\{ \lambda \sum_{j=1}^{k} \gamma_{j} \right\} \left( 1 + (  \sum_{j=1}^{k} \gamma_{j} )^{ 2 \theta}\right)} + 2 \lambda^{ 2 \theta } \left\| g^{*} \right\|_{2}^2 +\\
                &  ( 1 + \kappa_1^2 \kappa_2^2 )^2 ( 1 + c )\sum_{i=1}^{k}  \frac{ \gamma_{i}^2 \exp \left\{ - \lambda \sum_{j=i+1}^{k} \gamma_{j} \right\} }{  1 + \sum_{j=i+1}^{k} \gamma_{j}} \max_{1\leqslant j\leqslant k} \left( \sigma^2 + \mathbb{E}_{Z^{i-1} } [ \mathcal{E} (\hat{\eta}_{i}) ] - \mathcal{E}( \eta^{*}) \right).
            \end{aligned}
        \end{equation}
    \end{proposition}
    \begin{proof}
        We begin with the estimation of $ \| L_{C}^{1 / 2 }  ( \beta_{\lambda} - \beta^{*}) \|_{2} $ in Proposition \ref{prop:errorDECOM}.
        Assumption \ref{assum:regularityBeta} with $0<\theta\leqslant 1$ and (\ref{eq:betalambdaExpicit}) imply that
        \begin{equation}
            \label{eq:DECOMbetaLambdaStar}
            \begin{aligned}
                L_{C}^{1 / 2 } ( \beta_{\lambda} - \beta^{*}) &= (\lambda I + T_{K})^{-1} T_{K} L_{C}^{1 / 2 }  \beta^{*} - L_{C}^{1 / 2 }  \beta^{*} \\
                &= - \lambda \left( \lambda I + T_{K} \right)^{-1} L_{C}^{1 / 2 } \beta^{*} = -\lambda \left( \lambda I + T_{K} \right)^{-1} T_{K}^{\theta} g^{*},
            \end{aligned}
        \end{equation}
        and
        \begin{equation}
            \label{eq:BoundNormBetaLambdaStar}
            \begin{aligned}
                \left\| L_{C}^{1 / 2 } ( \beta_{\lambda} - \beta^{*}) \right\|_{2} & \leqslant \lambda \left\| \left( \lambda I + T_{K} \right)^{-1} T_{K}^{\theta} \right\| \left\| g^{*} \right\|_{2} \leqslant  \lambda^{ \theta} \left\| g^{*} \right\|_{2},
            \end{aligned}
        \end{equation}
        where the second inequality in (\ref{eq:BoundNormBetaLambdaStar}) follows from that $ \| \left( \lambda I + T_{K} \right)^{-1} T_{K}^{\theta} \| \leqslant \lambda^{\theta-1}   $ for any $ 0 < \theta \leqslant 1 $.
        In addition, since $ \| \left( \lambda I + T_{K} \right)^{-1} T_{K} \| \leqslant 1 $, we can bound $ \| \omega_{1}^{k} ( T_{K} + \lambda I) L_{C}^{1 / 2 } \beta_{\lambda} \|_{2} $ by Lemma \ref{lem:betalambda} and Assumption \ref{assum:regularityBeta} as
        \begin{equation}
            \label{eq:BoundNormBetaLambda}
            \begin{aligned}
                \left\| \omega_{1}^{k} ( T_{K} + \lambda I) L_{C}^{1 / 2 } \beta_{\lambda} \right\|_{2} & \leqslant  \left\| \omega_{1}^{k} ( T_{K} + \lambda I) \left( \lambda I + T_{K} \right)^{-1} T_{K}^{1 + \theta} g^{*} \right\|_{2} \\
                & \leqslant  \left\| \omega_{1}^{k} ( T_{K} + \lambda I) T_{K}^{\theta} \right\| \left\| g^{*} \right\|_{2}.
            \end{aligned}
        \end{equation}
        Combining (\ref{eq:BoundNormBetaLambdaStar}) and (\ref{eq:BoundNormBetaLambda}) with Proposition \ref{prop:errorDECOM} and applying Lemma \ref{lem:prodOperNormBound} to $ \| \omega_{1}^{k} ( T_{K} + \lambda I) T_{K}^{\theta} \|^2 $, we complete the proof by applying the inequality  $ ( a + b )^2 \leqslant 2 a^2 + 2 b^2 $ for any $ a,b>0 $.
    \end{proof}

    Besides, following the same spirit as that in the proof of Proposition \ref{prop:errorDECOM2}, we present the following proposition to bound the excess prediction error with the capacity assumption.

    \begin{proposition}
        \label{prop:errorDECOM2capa}
        Let $ \{ \hat{\eta}_{k+1}=\langle\beta_{k+1},\cdot\rangle: k \in [n]\} $ be defined through (\ref{eq:SGDalgorithm}).
        If $ \gamma_{k}(\kappa_1^2 \kappa_2^2+\lambda) \leqslant 1 $ for any $ k \in [n] $, then under Assumption \ref{assum:bounded4thmoment}, Assumption \ref{assum:regularityBeta} with $0<\theta \leqslant  1$ and Assumption \ref{assum:capacity} with $ 0 < s < 1 $, we have
        \begin{equation}
            \label{eq:errorDECOM2capa}
            \begin{aligned}
                \mathbb{E}_{Z^{k} }[\mathcal{E} ( \hat{\eta}&_{k+1})] - \mathcal{E}( \eta^{*})  \leqslant  2 \left\| g^{*} \right\|_{2}^2 \frac{ \left( \theta / e \right)^{2 \theta} + ( \kappa_1 \kappa_2)^{4 \theta} }{ \exp \left\{ \lambda \sum_{j=1}^{k} \gamma_{j} \right\} \left( 1 + (  \sum_{j=1}^{k} \gamma_{j} )^{ 2 \theta}\right)} + 2 \lambda^{ 2 \theta } \left\| g^{*} \right\|_{2}^2 +\\
                &  (\sqrt{c} + c) \operatorname{Tr} (T_{K}^{s} ) \sum_{i=1}^{k} \frac{ \gamma_{i}^2 \left( \sigma^2 +  \mathbb{E}_{Z^{i-1}} [\mathcal{E} (\hat{\eta}_{i})] - \mathcal{E}( \eta^{*}) \right) \left[ \left( \frac{2-s}{2\mathrm{e}} \right)^{2-s} + (\kappa_1 \kappa_2)^{4-2s} \right] }{ \exp \left\{  \lambda \sum_{j=i+1}^{k} \gamma_{j} \right\} \left( 1 + (\sum_{j=i+1}^{k} \gamma_{j}) ^{2-s} \right)}.
            \end{aligned}
        \end{equation}
    \end{proposition}

    \noindent
    \textbf{Proof of Theorem \ref{thm:1}.} We use Lemma \ref{lem:seriesgamma2}, Proposition \ref{prop:UniformBounds} and \ref{prop:errorDECOM2} to derive explicit error rates for $ \mathbb{E}_{Z^{k} }[\mathcal{E} ( \hat{\eta}_{k+1})]- \mathcal{E}( \eta^{*}), \, k \in [n] $ with appropriate step-size.
    Let $ J_1^{*} $, $ J_2^{*} $ and  $ J_{3}^{*} $ denote the three terms on the right-hand side of (\ref{eq:errorDECOM2}), respectively.
    Recall that $ \gamma_{j} = \gamma_1 j^{-\mu} $ with $\gamma_1 >0 $ and $ 0 < \mu <1 $ for any $ j \in [k] $.
    Then we have that for any $ k \in [n] $,
    \begin{equation}
        \label{eq:sumGammaj}
        \sum_{j=1}^{k} \gamma_{j} = \gamma_1 \sum_{j=1}^{k} j^{-\mu} \geqslant \frac{\gamma_1}{1-\mu} \left[ ( k+1 )^{1-\mu} - 1 \right] \geqslant \gamma_1 d_{\mu} (k+1)^{1-\mu},
    \end{equation}
    where $ d_{\mu} = ( 1-2^{\mu-1}  ) / ( 1-\mu )$.
    Then $J_1^{*}$ can be bounded as
    \begin{equation}
        \label{eq:J1starBound}
        \begin{aligned}
            J_1^{*} & \leqslant 2\left\| g^{*} \right\|_{2}^2 \left[ \left( \theta / e \right)^{2 \theta} + ( \kappa_1 \kappa_2)^{4 \theta} \right] \exp \left\{ - \lambda \gamma_1 d_{\mu}  (k+1)^{1-\mu} \right\} \gamma_1^{-2\theta} d_{\mu}^{-2\theta} \left( k+1 \right)^{-2 \theta ( 1-\mu )} \\
            & \leqslant  C^{\mathsf{ol}}_{\mathsf{p}1}  \gamma_1^{-2\theta} \exp \left\{ -  \lambda \gamma_1 d_{\mu}  k^{1-\mu} \right\} k ^{-2 \theta ( 1-\mu )},
        \end{aligned}
    \end{equation}
    where $ C^{\mathsf{ol}}_{\mathsf{p}1} = 2d_{\mu}^{-2\theta} \left\| g^{*} \right\|_{2}^2 \left[ \left( \theta / e \right)^{2 \theta} + ( \kappa_1 \kappa_2)^{4 \theta} \right] $.
    Proposition \ref{prop:UniformBounds} and Corollary \ref{coro:seriesConstantBound} show that if $ \gamma_1 \leqslant \left[ 4 C_{\mu} (1+c) ( 1 + \kappa_1^2 \kappa_2^2 )^2 ( \log 2 + \min \{ \mu, 1-\mu \}^{-1} ) \right]^{-1}, $ we have $ \mathbb{E}_{Z^{k} }[\mathcal{E} ( \hat{\eta}_{k+1})]- \mathcal{E}( \eta^{*}) \leqslant 8 \kappa_{2}^2 \left\| \beta^{*} \right\|_{2}^2 + \sigma^2 $ for any $ k \in [n] $.
    Hence, we can use Lemma \ref{lem:seriesgamma2} with $ \nu = 1 $ to derive
    \begin{equation}
        \label{eq:J4starBound}
        \begin{aligned}
            J_{3}^{*} & \leqslant ( 1 + \kappa_1^2 \kappa_2^2 )^2 ( 1 + c ) \left( 8 \kappa_{2}^2 \left\| \beta^{*} \right\|_{2}^2 + 2 \sigma^2 \right) C_{\mu} \gamma_1 \cdot\\
            & \qquad \left( \exp \left\{ -\lambda \gamma_1 d_{\mu} k^{1-\mu} \right\} k^{ -\min\{ \mu,1-\mu \} } + k^{-\mu} \right) \log(k+1) \\
            & \leqslant C^{\mathsf{ol}}_{\mathsf{p}3} \left( \exp \left\{ -\lambda \gamma_1 d_{\mu} k^{1-\mu} \right\} k^{ -\min\{ \mu,1-\mu \} } + k^{-\mu} \right) \log(k+1) ,
        \end{aligned}
    \end{equation}
    where $ C^{\mathsf{ol}}_{\mathsf{p}3} = ( 4 \kappa_{2}^2 \left\| \beta^{*} \right\|_{2}^2 + \sigma^2 ) [ 2 ( \log 2 + \min \{ \mu, 1-\mu \}^{-1} ) ]^{-1} $.
    Putting (\ref{eq:J1starBound}) and (\ref{eq:J4starBound}) back into (\ref{eq:errorDECOM2}) in Proposition \ref{prop:errorDECOM2}, we have that for $ k \in [n] $,
    \begin{equation}
        \label{eq:predErrorFinal}
        \begin{aligned}
            \mathbb{E}_{Z^{k} }[\mathcal{E} ( \hat{\eta}_{k+1})] -  \mathcal{E}( \eta^{*}) \leqslant& C^{\mathsf{ol}}_{\mathsf{p}1} \gamma_1^{-2\theta} \exp \left\{ -  \lambda \gamma_1 d_{\mu}  k^{1-\mu} \right\} k^{-2 \theta ( 1-\mu )} + 2 \lambda^{ 2\theta } \left\| g^{*} \right\|_{2}^2 \\
            &  + C^{\mathsf{ol}}_{\mathsf{p}3} \left( \exp \left\{ -\lambda \gamma_1 d_{\mu} k^{1-\mu} \right\} k^{ -\min\{ \mu,1-\mu \} } + k^{-\mu} \right) \log(k+1).
        \end{aligned}
    \end{equation}
    When $ \theta \leqslant \frac{1}{2} $, we have $ \mu \leqslant \frac{1}{2} $ from $ \mu = \frac{2 \theta}{2 \theta +1} $.
    Hence, we can set $ k = n $ and $ \lambda = n^{-\frac{1}{2\theta +1}} $ to obtain
    \begin{equation*}
        \mathbb{E}_{Z^{n} }[\mathcal{E} ( \hat{\eta}_{n+1})] - \mathcal{E}( \eta^{*}) \lesssim n^{-\frac{2 \theta}{2 \theta + 1}} \log(n+1).
    \end{equation*}
    When $ \theta >\frac{1}{2} $, we have $ \mu >\frac{1}{2} $.
    Note that for any $ \varepsilon > 0, $ $ a > 0, $ and $ \omega > 0 $, the following asymptotic behavior holds
    \begin{equation}
        \exp \{ - a k^{\varepsilon} \} = \mathcal{O}(k^{-\omega}).
    \end{equation}
    Hence, for any $ 0 < \varepsilon < 2 \theta / (2 \theta + 1) $, it can be derived by choosing $ k = n $ and $ \lambda =  n^{-\frac{1}{2 \theta + 1} + \frac{\varepsilon}{2 \theta}} $ that
    \begin{equation*}
        \mathbb{E}_{Z^{n} }[\mathcal{E} ( \hat{\eta}_{n+1})] - \mathcal{E}( \eta^{*}) \lesssim n^{-\frac{2 \theta}{2 \theta + 1} + \varepsilon}.
    \end{equation*}
\qed

    \noindent
    \textbf{Proof of Theorem \ref{thm:2}.} To prove Theorem \ref{thm:2}, we shall combine the assumption that $ 0 < \gamma_1 \leqslant \min\left\{(1+\kappa_1^2 \kappa_2^2)^{-1}, C_{1}^{\mathsf{S}} \right\} $ and Corollary \ref{coro:seriesConstantBoundcapa} together to verify the assumptions in Proposition \ref{prop:UniformBoundscapa}.
    Then, we plug the uniform bounds (\ref{eq:UniformBoundscapa}), (\ref{eq:J1starBound}) into (\ref{eq:errorDECOM2capa}) in Proposition \ref{prop:errorDECOM2capa} and apply Lemma \ref{lem:seriesgamma2} with $ \nu = 2-s $ to obtain that for $ k \in[n] $,
    \begin{equation}
        \label{eq:predErrorFinalCapa}
        \begin{aligned}
            \mathbb{E}_{Z^{k} }[\mathcal{E} ( \hat{\eta}_{k+1})] &-  \mathcal{E}( \eta^{*}) \leqslant C^{\mathsf{ol}}_{\mathsf{p}1} \gamma_1^{-2\theta} \exp \left\{ -  \lambda \gamma_1 d_{\mu}  k^{1-\mu} \right\} k^{-2 \theta ( 1-\mu )} + 2 \lambda^{ 2\theta } \left\| g^{*} \right\|_{2}^2 \\
            & + \frac{1}{2}( 4 \kappa_{2}^2 \left\| \beta^{*} \right\|_{2}^2 + \sigma^2 ) \left( \exp \left\{ - \lambda \gamma_1 d_{\mu} k^{1-\mu} \right\} k^{-\min \{ \mu, (2-s)(1-\mu) \}} + k^{-\mu}  \right).
        \end{aligned}
    \end{equation}

    When $ 2\theta \leqslant 2-s $, we have $ \mu = \frac{2\theta}{2\theta+1} \leqslant \frac{2-s}{3-s} $ which deduces that $ \mu \leqslant (2-s)(1-\mu) $.
    Hence, $ \mathbb{E}_{Z^{n} }\left[\mathcal{E} ( \hat{\eta}_{n+1})\right]- \mathcal{E}( \eta^{*})  \lesssim n^{-\frac{2\theta}{2\theta+1} } $ with $ \lambda = n^{-\frac{1}{2 \theta + 1}} $.

    When $ 2\theta \geqslant 2-s $, we have $ \mu = \frac{2-s}{3-s} $ and $ \mathbb{E}_{Z^{n} }\left[\mathcal{E} ( \hat{\eta}_{n+1})\right]- \mathcal{E}( \eta^{*})  \lesssim n^{-\frac{2-s}{3-s} } $ with $ \lambda = n^{-\frac{1}{3-s}} $.
    It can be also derived that $ \mathbb{E}_{Z^{n} }\left[\mathcal{E} ( \hat{\eta}_{n+1})\right]- \mathcal{E}( \eta^{*})  \lesssim n^{-\frac{2\theta}{2\theta+1}+ \varepsilon } $ with $ \mu = \frac{2\theta}{2\theta+1} $ and $ \lambda = n^{-\frac{1}{2\theta+1}+\frac{\varepsilon}{2\theta}} $ for any $ 0 < \varepsilon < \frac{2\theta}{2\theta+1} $. \qed

    \noindent
    \textbf{Proof of Theorem \ref{thm:3}.} In the following, we will prove Theorem \ref{thm:3} for the benign case (Assumption \ref{assum:capacity} with $s=1$) and the stringent case ($ 0<s<1 $).
    This is achieved by bounding the three terms in Proposition \ref{prop:errorDECOM2} and \ref{prop:errorDECOM2capa} with $ k = n $, respectively.
    Recall that the constant step-sizes are of the form $ \{ \gamma_{k} = \gamma, k \in [n] \} $ with $ \gamma = \gamma_{0} n^{-\mu} $, then we have $ \sum_{j=1}^{n} \gamma_{j} = \gamma n = \gamma_0 n^{1 - \mu}, $ and
    \begin{equation}
        \label{eq:ConsSeriesInequality}
        \begin{aligned}
            \sum_{i=1}^{n}  \frac{ \gamma_{i}^2 \exp \left\{ - \lambda \sum_{j=i+1}^{n} \gamma_{j} \right\} }{  1 + \left( \sum_{j=i+1}^{n} \gamma_{j} \right)^{2-s}  } & = \gamma^2 \sum_{i=1}^{n} \frac{ \exp \left\{ - \lambda \gamma (n-i) \right\} }{  1 + [\gamma (n-i)]^{2-s} } = \gamma^2 + \gamma^2 \sum_{i=1}^{n-1} \frac{ \mathrm{e}^{- \lambda \gamma i} }{  1 + (\gamma i)^{2-s} } \\
            & \leqslant \gamma^2 + \gamma \int_{0}^{n-1} \frac{ \gamma }{  1 + (\gamma x)^{2-s} } \mathrm{d} x \leqslant \gamma^2 + \gamma \int_{0}^{\gamma n} \frac{ 1 }{  1 + u^{2-s} } \mathrm{d} u \\
            &\leqslant \gamma^2 + \gamma
            \begin{cases}
                \frac{2-s}{1-s}, & 0 < s < 1, \\
                \log (1 + \gamma n), & s=1,
            \end{cases}
        \end{aligned}
    \end{equation}
    where we used that for $ s <1 $, $ \int_{0}^{\gamma n} \frac{ 1 }{  1 + u^{2-s} } \mathrm{d} u \leqslant 1 + \int_{1}^{\gamma n} u^{s-2}  \mathrm{d} u \leqslant 1 + \frac{1-(\gamma n)^{s-1} }{1-s} $.

    Let
    \[
        C_{2}^{\mathsf{S}} = [ 2 (1+c) ( 1 + \kappa_1^2 \kappa_2^2 )^2 (1+\mu^{-1})]^{-1},
    \]
    and
    \[
        C_{2*}^{\mathsf{S}} = \left\{ 2 ( \sqrt{c} +c ) \operatorname{Tr} (T_{K}^{s} ) \left[ \left( \frac{2-s}{2\mathrm{e}} \right)^{2-s} + (\kappa_1 \kappa_2)^{4-2s} \right] \left( 2 + \frac{1}{1-s} \right) \right\}^{-1}.
    \]
    Under the assumption that $ 0 < \gamma_0 \leqslant \min\{ C_{2}^{\mathsf{S}},C_{2*}^{\mathsf{S}} \} $, we can directly obtain that $ \gamma_0 < 1 $ and
    \[
        \gamma \log ( 1 + \gamma n) \leqslant n^{-\mu} \log (1 + n^{1-\mu} ) \leqslant n^{-\mu} ( \log 2 + (1-\mu) \log n ) \leqslant 1 + \mu^{-1}(1-\mu) = \mu^{-1},
    \]
    where the last inequality follows from that $ n^{-\mu} \log n \leqslant \mu^{-1} $ for any $ n \geqslant 1 $.
    Therefore, we have
    \begin{equation}
        \label{eq:ConsSeriesInequalityFinal}
        \sum_{i=1}^{n}  \frac{ \gamma_{i}^2 \exp \left\{ - \lambda \sum_{j=i+1}^{n} \gamma_{j} \right\} }{  1 + \sum_{j=i+1}^{n} \gamma_{j}} \leqslant \gamma_0 \times
        \begin{cases}
            2 + 1 / (1-s), & 0<s<1, \\
            1 + \mu^{-1},& s=1.
        \end{cases}
    \end{equation}
    Besides, it is easy to verify that $ \gamma_{k}(\kappa_1^2 \kappa_2^2+\lambda) \leqslant 1 $ for any $ k \in [n] $.
    The above inequality and the assumption on $ \gamma_0 $ guarantee the conditions in Proposition \ref{prop:UniformBounds} and \ref{prop:UniformBoundscapa}.
    Consequently, we have that for any $ k \in [n] $, $ \mathbb{E}_{Z^{k} }[\mathcal{E} ( \hat{\eta}_{k+1})]- \mathcal{E}( \eta^{*}) \leqslant 8 \kappa_{2}^2 \left\| \beta^{*} \right\|_{2}^2 + \sigma^2. $

    When $ s=1 $, note that (\ref{eq:ConsSeriesInequality}) also implies that
    \begin{equation}
        \sum_{i=1}^{n}  \frac{ \gamma_{i}^2 \exp \left\{ - \lambda \sum_{j=i+1}^{n} \gamma_{j} \right\} }{  1 + \sum_{j=i+1}^{n} \gamma_{j}} \leqslant \gamma + \gamma \log (n+1) \leqslant ( 1 + \log^{-1} 2) \gamma \log(n+1).
    \end{equation}
    Combining this with Proposition \ref{prop:errorDECOM2}, we have
    \begin{equation}
        \begin{aligned}
            \mathbb{E}_{Z^{n} }[\mathcal{E} ( \hat{\eta}_{n+1})] - \mathcal{E}( \eta^{*}) & \leqslant C^{\mathsf{fi}}_{\mathsf{p}1} \gamma_0^{-2\theta} \exp \{ - \lambda \gamma_0 n^{1-\mu} \} n^{-\mu}+ 2\lambda^{ 2 \theta } \left\| g^{*} \right\|_{2}^2 \\
            & \qquad  + C^{\mathsf{fi}}_{\mathsf{p}3} n^{-\mu} \log(n+1),
        \end{aligned}
    \end{equation}
    where
    \[
        \begin{aligned}
            C^{\mathsf{fi}}_{\mathsf{p}1} & = 2 \left\| g^{*} \right\|_{2}^2 \left[ \left( \theta / e \right)^{2 \theta} + ( \kappa_1 \kappa_2)^{4 \theta} \right], \\
            C^{\mathsf{fi}}_{\mathsf{p}3} & = \frac{ \mu}{1 + \mu} (\log^{-1} 2 + 1) \left( 4 \kappa_{2}^2 \left\| \beta^{*} \right\|_{2}^2 + \sigma^2 \right).
        \end{aligned}
    \]
    Finally, when $ \lambda = n^{- \frac{1}{2 \theta +1}} $ and $ \mu = \frac{2\theta}{2 \theta +1} $, we have
    \begin{equation*}
        \mathbb{E}_{Z^{n} }[\mathcal{E} ( \hat{\eta}_{n+1})] - \mathcal{E}( \eta^{*}) \leqslant C^{\mathsf{fi}}_{\mathsf{p}} n^{-\frac{2 \theta}{2 \theta +1}} \log (n+1),
    \end{equation*}
    where $ C^{\mathsf{fi}}_{\mathsf{p}} = 2 \left\| g^{*} \right\|_{2}^2 + \mathrm{e}^{-\gamma_0} \gamma_0^{-2\theta} C^{\mathsf{fi}}_{\mathsf{p1}} + C^{\mathsf{fi}}_{\mathsf{p3}} $.

    When $ 0 < s <1 $, combining (\ref{eq:ConsSeriesInequality}) and Proposition \ref{prop:errorDECOM2capa} yields that
    \begin{equation}
        \begin{aligned}
            \mathbb{E}_{Z^{n} }[\mathcal{E} ( \hat{\eta}_{n+1})] - \mathcal{E}( \eta^{*}) & \leqslant C^{\mathsf{fi}}_{\mathsf{p}1} \gamma_0^{-2\theta} \exp \{ - \lambda \gamma_0 n^{1-\mu} \} n^{-\mu} + 2\lambda^{ 2 \theta } \left\| g^{*} \right\|_{2}^2 \\
            & \qquad + \left( 4 \kappa_{2}^2 \left\| \beta^{*} \right\|_{2}^2 + \sigma^2 \right) n^{-\mu}.
        \end{aligned}
    \end{equation}
    Substituting $ \lambda = n^{- \frac{1}{2 \theta +1}} $ and $ \mu = \frac{2\theta}{2 \theta +1} $ into the above inequality, we have
    \begin{equation*}
        \mathbb{E}_{Z^{n} }[\mathcal{E} ( \hat{\eta}_{n+1})] - \mathcal{E}( \eta^{*}) \leqslant C^{\mathsf{fi}}_{\mathsf{p,c}} n^{-\frac{2 \theta}{2 \theta +1}},
    \end{equation*}
    where $ C^{\mathsf{fi}}_{\mathsf{p,c}} = 2 \left\| g^{*} \right\|_{2}^2 + \mathrm{e}^{-\gamma_0} \gamma_0^{-2\theta} C^{\mathsf{fi}}_{\mathsf{p1}} + 4 \kappa_{2}^2 \left\| \beta^{*} \right\|_{2}^2 + \sigma^2 $. \qed

    \begin{remark}
        The cumulative sample error (3rd term) does not include $ n^{1-\mu} $, which leads to a better convergence rate for the excess prediction error without $ \varepsilon $ included in $ \lambda $ even if for $ \mu \geqslant 1 / 2  $, i.e., $ \theta \geqslant 1 / 2  $.
    \end{remark}

\section{Analysis for the Estimation Error}
\label{sec:proofOfthm34}

    In this section, we bound the estimation error of $ \beta_{k+1} $ in algorithm (\ref{eq:SGDalgorithm}), i.e., $ \mathbb{E}_{Z^{k} } \left\| \beta_{k+1} - \beta^{*} \right\|_{K}^2 $, and derive explicit error rates with specific step-sizes, which proves Theorem \ref{thm:4} and \ref{thm:5}.
    To this end, we first employ Assumption \ref{assum:regularityBeta1} to estimate the approximation errors in Proposition \ref{prop:errorDECOMKNorm1} and \ref{prop:errorDECOMKNormwithCapacity}, respectively.

    \begin{proposition}
        \label{prop:errorDECOMKNorm2}
        Define $ \{ \beta_{k+1}: k \in [n]\} $ by (\ref{eq:SGDalgorithm}) with $ \lambda > 0 $.
        If $ \gamma_{k}(\kappa_1^2 \kappa_2^2+\lambda) \leqslant 1 $ for any $ k \in [n] $, then under Assumption \ref{assum:bounded4thmoment} and Assumption \ref{assum:regularityBeta1} with $0<r\leqslant 1,$ we have
        \begin{equation}
            \label{eq:errorDECOMKNorm2}
            \begin{aligned}
                \mathbb{E}_{Z^{k} } \left\| \beta_{k+1} - \beta^{*} \right\|&_{K}^2 \leqslant  2 \left\| g_{*} \right\|_{2}^2 \frac{ \left( r / e \right)^{2 r} + ( \kappa_1 \kappa_2)^{4 r} }{ \exp \left\{ \lambda \sum_{j=1}^{k} \gamma_{j} \right\} \left( 1 + (  \sum_{j=1}^{k} \gamma_{j} )^{ 2 r}\right)} + 2 \lambda^{ 2 r } \left\| g_{*} \right\|_{2}^2\\
                + &\kappa_1^2 \kappa_2^2 (1+c)  \sum_{i=1}^{k}  \frac{ \gamma_{i}^2   }{ \exp \left\{  \lambda \sum_{j=i+1}^{k} \gamma_{j} \right\} } \max_{1\leqslant i\leqslant k} \left( \sigma^2 + \mathbb{E}_{Z^{i-1} } [ \mathcal{E} (\hat{\eta}_{i}) ] - \mathcal{E}( \eta^{*}) \right).
            \end{aligned}
        \end{equation}
    \end{proposition}

    \begin{proof}
        First, recall that $ f_{\lambda} = (\lambda I + T_{C})^{-1} L_K^{1/2} L_{C} \beta^{*} $ in (\ref{eq:flambda}), then under Assumption \ref{assum:regularityBeta1} with $0<r \leqslant 1$ (i.e., $ \beta^{*} = L_{K}^{1 / 2 } T_{C}^{r} g_{*} $), we have
        \begin{equation}
            \label{eq:BoundKNorm2}
            \begin{aligned}
                \left\| L_{K}^{1 / 2 } f_{\lambda} - \beta^{*} \right\|_{K} & = \left\| f_{\lambda} - T_{C}^{r} g_{*} \right\|_{2} = \left\| (\lambda I + T_{C})^{-1} T_{C}^{1+r} g_{*} - T_{C}^{r} g_{*} \right\|_{2} \\
                & = \lambda \left\| (\lambda I + T_{C})^{-1} T_{C}^{r} g_{*} \right\|_{2}  \leqslant \lambda \left\| (\lambda I + T_{C})^{-1} T_{C}^{r} \right\| \left\| g_{*} \right\|_{2} \leqslant  \lambda^{r} \left\| g_{*} \right\|_{2},
            \end{aligned}
        \end{equation}
        where the last inequality above follows from that $ \| \left( \lambda I + T_{C} \right)^{-1} T_{C}^{r} \| \leqslant \lambda^{r-1}   $ for any $ 0 < r \leqslant 1 $.
        Further, it is not hard to derive that
        \begin{equation}
            \label{eq:BoundKNorm1}
            \begin{aligned}
                \left\| \omega_{1}^{k} ( L_{K}L_{C} + \lambda I) L_{K}^{1 / 2 } f_{\lambda} \right\|_{K} & = \left\| L_{K}^{1 / 2 } \omega_{1}^{k} ( T_{C} + \lambda I) f_{\lambda} \right\|_{K} \\
                & = \left\|  \omega_{1}^{k} ( T_{C} + \lambda I) (\lambda I + T_{C})^{-1} T_{C}^{1+r} g_{*} \right\|_{2} \\
                & \leqslant \left\| \omega_{1}^{k} ( T_{C} + \lambda I) T_{C}^{r} \right\| \left\| g_{*} \right\|_{2},
            \end{aligned}
        \end{equation}
        where the last inequality follows from $ (\lambda I + T_{C})^{-1} T_{C}^{r} = T_{C}^{r} (\lambda I + T_{C})^{-1} $ and $ \| (\lambda I + T_{C})^{-1} T_{C} \| \leqslant 1 $.
        We complete the proof by combining (\ref{eq:BoundKNorm2}) and (\ref{eq:BoundKNorm1}) with Proposition \ref{prop:errorDECOMKNorm1} and applying Lemma \ref{lem:prodOperNormBound} to $ \| \omega_{1}^{k} ( T_{C} + \lambda I) T_{C}^{r} \|^2 $.
    \end{proof}

    Analogously, we can obtain the following result for the estimation error under the additional Assumption \ref{assum:capacity}, which is induced by Proposition \ref{prop:errorDECOMKNormwithCapacity}.

    \begin{proposition}
        \label{prop:errorKNormwithCapacity}
        Define $ \{ \beta_{k+1}: k \in [n]\} $ by (\ref{eq:SGDalgorithm}) with $ \lambda > 0 $.
        If $ \gamma_{k}(\kappa_1^2 \kappa_2^2+\lambda) \leqslant 1 $ for any $ k \in [n] $, then under Assumption \ref{assum:bounded4thmoment}, Assumption \ref{assum:capacity} with $ 0 < s < 1 $ and Assumption \ref{assum:regularityBeta1} with $ 0 < r \leqslant 1 $ we have
        \begin{equation}
            \label{eq:errorKNormwithCapacity}
            \begin{aligned}
                \mathbb{E}_{Z^{k} }&\left\| \beta_{k+1} - \beta^{*} \right\|_{K}^2 \leqslant  2 \left\| g_{*} \right\|_{2}^2 \frac{ \left( r / e \right)^{2 r} + ( \kappa_1 \kappa_2)^{4 r} }{ \exp \left\{ \lambda \sum_{j=1}^{k} \gamma_{j} \right\} \left( 1 + (  \sum_{j=1}^{k} \gamma_{j} )^{ 2 r}\right)} + 2 \lambda^{ 2 r } \left\| g_{*} \right\|_{2}^2 + \sqrt{c} \operatorname{Tr} (T_{C}^{s} ) \cdot \\
                & \max_{1\leqslant i\leqslant k} \left[ \sigma^2 + \sqrt{c} \left( \mathbb{E}_{Z^{i-1}} \mathcal{E} (\hat{\eta}_{i}) - \mathcal{E}( \eta^{*}) \right) \right] \sum_{i=1}^{k}\frac{ \gamma_{i}^2 \left[ \left( \frac{1-s}{2\mathrm{e}} \right)^{1-s} + (\kappa_1 \kappa_2)^{2(1-s)}   \right]   }{ \exp \left\{  \lambda \sum_{j=i+1}^{k} \gamma_{j} \right\} \left( 1 + (\sum_{j=i+1}^{k} \gamma_{j}) ^{1-s} \right) }.
            \end{aligned}
        \end{equation}
    \end{proposition}

    \noindent \textbf{Proof of Theorem \ref{thm:4}.} We prove Theorem \ref{thm:4} with polynomially decaying step-size in two cases including $ s = 1 $ and $ 0<s < 1 $ in Assumption \ref{assum:capacity}.

    When $ s = 1 $, recall that $ \gamma_{i} = \gamma_1 i^{-\mu} $, $ i \in [k] $ with $ \gamma_1 >0 $ and $ \mu=\frac{2r+1}{2r+2} $.
    Assumption \ref{assum:regularityBeta1} with $ r>0 $ implies that $ \mu > 1 / 2  $ and Lemma \ref{lem:seriesgamma2} provides an upper bound for the series in Proposition \ref{prop:errorDECOMKNorm2},
    \[
        \sum_{i=1}^{k} \gamma_{i}^2 \exp \left\{ - \lambda \sum_{j=i+1}^{k} \gamma_{j} \right\} \leqslant \widetilde{C}_{\mu}\gamma_1 \left( \exp \left\{ - \lambda \gamma_1 d_{\mu} (k+1)^{1-\mu} \right\} + k^{1-2\mu} \right).
    \]
    Combining the assumption $ \gamma_1 \leqslant \left[ 4 C_{\mu} (1+c) ( 1 + \kappa_1^2 \kappa_2^2 )^2 ( \log 2 + \min \{ \mu, 1-\mu \}^{-1} ) \right]^{-1} $ with Proposition \ref{prop:UniformBounds} and Corollary \ref{coro:seriesConstantBound}, we obtain $ \mathbb{E}_{Z^{k} }[\mathcal{E} ( \hat{\eta}_{k+1})]- \mathcal{E}( \eta^{*}) \leqslant 8 \kappa_{2}^2 \left\| \beta^{*} \right\|_{2}^2 + \sigma^2 $ for any $ k \in [n] $.
    Therefore, we apply Proposition \ref{prop:errorDECOMKNorm2} to obtain
    \begin{equation}
        \begin{aligned}
            \mathbb{E}_{Z^{k} }\left\| \beta_{k+1} - \beta^{*} \right\|_{K}^2 & \leqslant C^{\mathsf{ol}}_{\mathsf{e}1} \gamma_1^{-2r} \exp \left\{ -  \lambda \gamma_1 d_{\mu}  k^{1-\mu} \right\} k ^{-2 r ( 1-\mu )} + 2 \lambda^{ 2 r } \left\| g_{*} \right\|_{2}^2 \\
            & \quad + C^{\mathsf{ol}}_{\mathsf{e}3} \left( \exp \left\{ - \lambda \gamma_1 d_{\mu}  k^{1-\mu} \right\} + k^{1-2\mu} \right),
        \end{aligned}
    \end{equation}
    where
    \[
        \begin{aligned}
            C^{\mathsf{ol}}_{\mathsf{e}1} & = 2 d_{\mu}^{-2r} \left\| g_{*} \right\|_{2}^2 \left[ \left( r / e \right)^{2 r} + ( \kappa_1 \kappa_2)^{4 r} \right], \\
            C^{\mathsf{ol}}_{\mathsf{e}3} & = \widetilde{C}_{\mu} ( 4 \kappa_{2}^2 \left\| \beta^{*} \right\|_{2}^2 + \sigma^2 ) \left[ 2C_{\mu}(1+\kappa_1^2 \kappa_2^2) \left( \log 2 + (1-\mu)^{-1} \right) \right]^{-1}.
        \end{aligned}
    \]
    Finally, for any $ 0 < \varepsilon <  r / ( r + 1) $, it can be derived by choosing $ k=n $ and $ \lambda =  n^{-\frac{1}{2 r + 2} + \frac{\varepsilon}{2 r}} $ that
    \begin{equation*}
        \mathbb{E}_{Z^{n} }\left\| \beta_{n+1} - \beta^{*} \right\|_{K}^2 \lesssim   n^{-\frac{ r}{ r + 1} + \varepsilon} .
    \end{equation*}
    The proof is completed for the case $ s = 1 $.

    When $ 0<s < 1 $, recall that Lemma \ref{lem:seriesgamma2} with $ \nu=1-s $ provides a bound for the series in Proposition \ref{prop:errorKNormwithCapacity}, which is divided into two cases consisting of $ \mu < 1 /2  $ and $ \mu \geqslant 1 /2 $.
    In addition, we can combine Proposition \ref{prop:UniformBounds} and Corollary \ref{coro:seriesConstantBound} with the assumption $ \gamma_1 \leqslant \left[ 4 C_{\mu} (1+c) ( 1 + \kappa_1^2 \kappa_2^2 )^2 ( \log 2 + \min \{ \mu, 1-\mu \}^{-1} ) \right]^{-1} $ to guarantee the uniform bounds that $ \mathbb{E}_{Z^{k} }[\mathcal{E} ( \hat{\eta}_{k+1})]- \mathcal{E}( \eta^{*}) \leqslant 8 \kappa_{2}^2 \left\| \beta^{*} \right\|_{2}^2 + \sigma^2 $ for any $ k \in [n] $.
    Hence, explicit error rates can be derived by substituting these results into Proposition \ref{prop:errorKNormwithCapacity}.

    When $ 2r +s < 1 $, we have $ \mu < \frac{1}{2} $ from $ \mu = \frac{2 r + s}{ 2 r + s+ 1} $ and
    \begin{equation}
        \begin{aligned}
            \mathbb{E}_{Z^{k}} \left\| \beta_{k+1} - \beta^{*} \right\|_{K}^2 \leqslant C&^{\mathsf{ol}}_{\mathsf{e}1} \gamma_1^{-2r} \exp \left\{ -  \lambda \gamma_1 d_{\mu}  k^{1-\mu} \right\} k ^{-2 r ( 1-\mu )} + 2 \lambda^{ 2 r } \left\| g_{*} \right\|_{2}^2 \\
            &+ C^{\mathsf{ol}}_{\mathsf{e}3,\mathsf{c}} (k+1)^{-\mu + s(1-\mu)},
        \end{aligned}
    \end{equation}
    where $ C^{\mathsf{ol}}_{\mathsf{e}3,\mathsf{c}} = \sqrt{c} \hat{C}_{\mu} \operatorname{Tr} (T_{C}^{s} ) \left[ \left( \frac{1-s}{2\mathrm{e}} \right)^{1-s} + (\kappa_1 \kappa_2)^{2(1-s)} \right] \left( 8 \sqrt{c} \kappa_{2}^2 \left\| \beta^{*} \right\|_{2}^2 + (\sqrt{c} +1)\sigma^2 \right) $.
    Further, by setting $ k = n $ and $ \lambda =n^{-\frac{1}{2 r +s+ 1}} $, we have
    \begin{equation*}
        \mathbb{E}_{Z^{n} }\left\| \beta_{n+1} - \beta^{*} \right\|_{K}^2 \lesssim n^{-\frac{2 r}{2 r + s+ 1} }.
    \end{equation*}
    \indent When $ 2r + s \geqslant 1 $, we have $ \mu \geqslant  \frac{1}{2} $ and
    \begin{equation}
        \begin{aligned}
            \mathbb{E}_{Z^{k}} \left\| \beta_{k+1} - \beta^{*} \right\|&_{K}^2 \leqslant C^{\mathsf{ol}}_{\mathsf{e}1} \gamma_1^{-2r} \exp \left\{ -  \lambda \gamma_1 d_{\mu}  k^{1-\mu} \right\} k ^{-2 r ( 1-\mu )} + 2 \lambda^{ 2 r } \left\| g_{*} \right\|_{2}^2 \\
            + C^{\mathsf{ol}}_{\mathsf{e}3,\mathsf{c}}&[ (k+1)^{-\mu + s(1-\mu)} +\exp \left\{ - \lambda \gamma_1 d_{\mu} k^{1-\mu} \right\}k^{-(1-s)(1-\mu) } \log(k+1) ],
        \end{aligned}
    \end{equation}
    For any $ 0 < \varepsilon < 2 r / (2 r + s + 1) $, it can be derived by choosing $ k = n $ and $ \lambda =  n^{-\frac{1}{2 r +s+ 1} + \frac{\varepsilon}{2 r}} $ that
    \begin{equation*}
        \mathbb{E}_{Z^{n} }\left\| \beta_{n+1} - \beta^{*} \right\|_{K}^2 \lesssim n^{-\frac{2 r}{2 r + s+ 1} + \varepsilon}.
    \end{equation*}
    The proof of Theorem \ref{thm:4} is completed. \qed

    \noindent \textbf{Proof of Theorem \ref{thm:5}.}
    First recall that the step-sizes satisfy $ \gamma_{k} = \gamma = \gamma_{0} n^{-\mu} $ with $ 0 < \mu < 1 $ and $ 0 < \gamma_0 \leqslant [ 2 (1+c) ( 1 + \kappa_1^2 \kappa_2^2 )^2 (1 + \mu^{-1}) ]^{-1} $ for any $ k \in [n] $, then we have $ \gamma \leqslant 1 $ and $ \sum_{j=1}^{n} \gamma_{j} = n \gamma = \gamma_0 n^{1-\mu} $.
    Moreover, it follows from (\ref{eq:ConsSeriesInequalityFinal}) and Proposition \ref{prop:UniformBounds} that $ \mathbb{E}_{Z^{k} }[\mathcal{E} ( \hat{\eta}_{k+1})]- \mathcal{E}( \eta^{*}) \leqslant 8 \kappa_{2}^2 \left\| \beta^{*} \right\|_{2}^2 + \sigma^2 $ for any $ k \in [n] $.
    To bound the series in Proposition \ref{prop:errorDECOMKNorm2} and \ref{prop:errorKNormwithCapacity}, we have that for any $ 0 < s \leqslant 1 $,
    \begin{equation}
        \label{eq:ConsSeriesEqualityKnorm}
        \begin{aligned}
            \sum_{i=1}^{n}  \frac{ \gamma_{i}^2 \exp \left\{ - \lambda \sum_{j=i+1}^{n} \gamma_{j} \right\}}{  1 + \left( \sum_{j=i+1}^{n} \gamma_{j} \right)^{1-s} } & = \gamma^2 \sum_{i=1}^{n} \frac{ \exp \left\{ - \lambda \gamma (n-i) \right\} }{  1 + [\gamma (n-i)]^{1-s} } \leqslant \gamma^2 + \gamma^{1+s} \sum_{i=1}^{n-1} \frac{ \mathrm{e}^{- \lambda \gamma i} }{  1 + i^{1-s} }.
        \end{aligned}
    \end{equation}
    Since
    \[
        \begin{aligned}
            \sum_{i=1}^{n-1} \frac{ \mathrm{e}^{- \lambda \gamma i} }{  1 + i^{1-s} } & \leqslant \int_{0}^{n-1}  \frac{ 1 }{  1 + x^{1-s} } \mathrm{d}x \leqslant 1 + \int_{1}^{n-1}  \frac{ 1 }{  1 + x^{1-s} } \mathrm{d}x \\
            & \leqslant 1 + \int_{1}^{n-1} x^{s-1} \mathrm{d}x = 1 + \frac{1}{s} [ (n-1)^{s} -1 ] \leqslant s^{-1} n^{s},
        \end{aligned}
    \]
    combining the inequality above with (\ref{eq:ConsSeriesEqualityKnorm}) yields
    \begin{equation}
        \sum_{i=1}^{n}  \frac{ \gamma_{i}^2 \exp \left\{ - \lambda \sum_{j=i+1}^{n} \gamma_{j} \right\}}{  1 + \left( \sum_{j=i+1}^{n} \gamma_{j} \right)^{1-s} } \leqslant \frac{\gamma_0}{s} (n^{-\mu} + n^{-\mu(1+s) +s } ) \leqslant \frac{ 2 \gamma_0}{s} n^{-\mu + s(1-\mu)}.
    \end{equation}

    When $ s = 1 $, substituting these estimates into Proposition \ref{prop:errorDECOMKNorm2} yields that
    \[
        \mathbb{E}_{Z^{n} }\left\| \beta_{n+1} - \beta^{*} \right\|_{K}^2 \leqslant C^{\mathsf{fi}}_{\mathsf{e}1} \gamma_0^{-2r} \exp \{ - \lambda \gamma_0 n^{ 1-\mu} \} n^{2r(1-\mu)} + 2 \lambda^{ 2 r } \left\| g_{*} \right\|_{2}^2 + C^{\mathsf{fi}}_{\mathsf{e}3} n^{1-2\mu},
    \]
    where $ C^{\mathsf{fi}}_{\mathsf{e}1} = 2 \left\| g_{*} \right\|_{2}^2 \left[ \left( r / e \right)^{2 r} + ( \kappa_1 \kappa_2)^{4 r} \right] $ and  $ C^{\mathsf{fi}}_{\mathsf{e}3} = ( 8 \kappa_{2}^2 \left\| \beta^{*} \right\|_{2}^2 + 2\sigma^2 ) [ (1 + \kappa_1^2 \kappa_2^2) ( 1+ \mu^{-1} ) ]^{-1} $.

    When $ s<1 $, we can also obtain from Proposition \ref{prop:errorKNormwithCapacity} that
    \[
        \begin{aligned}
            \mathbb{E}_{Z^{n} }\left\| \beta_{n+1} - \beta^{*} \right\|_{K}^2 & \leqslant C^{\mathsf{fi}}_{\mathsf{e}1} \gamma_0^{-2r} \exp \{ - \lambda \gamma_0 n^{ 1-\mu} \} n^{2r(1-\mu)} + 2 \lambda^{ 2 r } \left\| g_{*} \right\|_{2}^2 + C^{\mathsf{fi}}_{\mathsf{e}3,\mathsf{c}} n^{-\mu + s(1-\mu)},
        \end{aligned}
    \]
    where $ C^{\mathsf{fi}}_{\mathsf{e}3,\mathsf{c}} = 2s^{-1} \operatorname{Tr} (T_{C}^{s} ) (c + \sqrt{c} ) ( 8 \kappa_{2}^2 \left\| \beta^{*} \right\|_{2}^2 + 2\sigma^2 ) $ $ \cdot \left[ \left( \frac{1-s}{2\mathrm{e}} \right)^{1-s} + (\kappa_1 \kappa_2)^{2(1-s)} \right] $.

    To conclude, the proof is completed with $ \mu = \frac{2 r + s}{ 2 r + s + 1} $, $ \lambda =  n^{-\frac{1}{2 r + s + 1}  } $ and $ C^{\mathsf{fi}}_{\mathsf{e}} = 2 \left\| g_{*} \right\|_{2}^2 + \mathrm{e}^{-\gamma_0} \gamma_0^{-2r} C^{\mathsf{fi}}_{\mathsf{e}1} + \max \{ C^{\mathsf{fi}}_{\mathsf{e}3}, C^{\mathsf{fi}}_{\mathsf{e}3,\mathsf{c}} \} $.\qed

\bibliography{wpref}
\bibliographystyle{abbrv}

\section{Appendix}
\label{sec:Appendix}
    This section includes some detailed proofs of lemmas and propositions used in the theoretical analysis of online regularized algorithm (\ref{eq:SGDalgorithm}) for functional linear model (\ref{eq:linearmodel}).

\subsection{Proof of Lemma \ref{lem:betalambda}}

    In this subsection, we shall prove Lemma \ref{lem:betalambda}.
    Since the optimization problem (\ref{eq:betalambda}) is strictly convex, the stationary point is the unique solution. A direct derivation with respect to $ \beta $ yields that
    \begin{equation*}
        \begin{aligned}
            &2 \mathbb{E}_{(X,Y)} \left( \left< \beta,X \right>_{2} - Y \right) \int_{\mathcal{T}} K(s,\cdot) X(s) ds + 2 \lambda \beta \\
            =& 2 \mathbb{E}_{X} \left( \left< \beta - \beta^{*},X \right>_{2} L_{K} X \right) + 2 \lambda \beta \\
            =& 2 L_{K} L_{C} ( \beta - \beta^{*} ) + 2 \lambda \beta = 0.
        \end{aligned}
    \end{equation*}
    Then, we have $  (\lambda I + L_{K} L_{C}) \beta_{\lambda} = L_{K} L_{C} \beta^{*} $.
    By left multiplying this equality with $ L_{C}^{1 / 2}  $ and combining the definition of $ T_{K} $ in (\ref{eq:LscrMscr}), we obtain that $ (\lambda I + T_{K}) L_{C}^{1 / 2 } \beta_{\lambda} =  T_{K} L_{C}^{1 / 2 } \beta^{*} $, which completes the proof.

\subsection{Proof of Lemma \ref{lem:prodOperNormBound}}

    In this subsection, we shall prove Lemma \ref{lem:prodOperNormBound}.
    Define the polynomial $ \tau(x) = x^{\alpha} \prod_{j=t}^{k} [1 - \gamma_{j}(x + \lambda)] $ on $ 0\leqslant x \leqslant C_{*} $.
    Since $ \gamma_{j}(C_{*}+\lambda) \leqslant 1 $ for any $ j \in [t,k] $, it follows that
    \begin{equation}
        \label{eq:tauxbound1}
        0 \leqslant \tau(x) \leqslant x^{\alpha} \prod_{j=t}^{k} ( 1 - \gamma_{j} \lambda ) \leqslant C_{*}^{\alpha}  \exp \left\{ - \lambda \sum_{j=t}^{k} \gamma_{j}  \right\},
    \end{equation}
    where we apply $ 1 - x \leqslant e^{-x} $ in the last inequality.
    On the other hand,
    \begin{equation}
        \label{eq:tauxbound2}
        \begin{aligned}
            \tau(x) &\leqslant  x^{\alpha} \exp \left\{ - (x + \lambda) \sum_{j=t}^{k} \gamma_{j} \right\} = x^{\alpha} \exp \left\{ - x \sum_{j=t}^{k} \gamma_{j} \right\} \cdot \exp \left\{ - \lambda \sum_{j=t}^{k} \gamma_{j} \right\} \\
            &\leqslant  \left( \frac{\alpha}{e} \right)^{\alpha} \left( \sum_{j=t}^{k} \gamma_{j} \right)^{-\alpha} \exp \left\{ - \lambda \sum_{j=t}^{k} \gamma_{j} \right\},
        \end{aligned}
    \end{equation}
    where the last step follows from that $ \max\limits_{x \geqslant 0} x^{\alpha} \mathrm{e}^{-Ax} = (\alpha / e)^{\alpha} A^{-\alpha}  $.

    Consequently, combining (\ref{eq:tauxbound1}) and (\ref{eq:tauxbound2}) implies that
    \begin{equation}
        \begin{aligned}
            \label{eq:tauxbound}
            \tau^2(x) \leqslant & \exp \left\{ - 2 \lambda \sum_{j=t}^{k} \gamma_{j} \right\}  \min \left\{ C_{*}^{2 \alpha} ,\left( \alpha / e \right)^{2 \alpha} \left( \sum_{j=t}^{k} \gamma_{j} \right)^{- 2 \alpha} \right\} \\
            \leqslant & \exp \left\{ -  \lambda \sum_{j=t}^{k} \gamma_{j} \right\} \frac{ \left( \alpha / e \right)^{2 \alpha} + C_{*}^{2 \alpha} }{ 1 + \left(  \sum_{j=t}^{k} \gamma_{j} \right)^{ 2 \alpha}},
        \end{aligned}
    \end{equation}
    where we apply the inequality $ \min \{ a, bc\} \leqslant \frac{ac}{1+c} + \frac{bc}{1+c} $ for any $ a,b,c>0 $ in the last step.
    Since $ \mathcal{A} $ is compact and positive, the proof is completed by combining (\ref{eq:tauxbound}) with the spectral theorem and the definition of $ \omega_{t}^{k} (\mathcal{A}+ \lambda I) $.

\subsection{Proof of Proposition \ref{prop:UniformBounds}}

    In this subsection, we shall prove Proposition \ref{prop:UniformBounds} by mathematical induction. Recall that $ \beta_1 = 0 $, then if follows from (\ref{eq:DECOMExcessRisk}) that
    \[
        \mathcal{E} ( \hat{\eta}_{1}) - \mathcal{E}( \eta^{*}) = \left\| L_{C}^{1 / 2 } \beta^{*} \right\|_{2}^2 \leqslant \left\| L_{C}^{1 / 2 }  \right\|^2 \left\| \beta^{*} \right\|_{2}^2 \leqslant \kappa_2^2 \left\| \beta^{*} \right\|_{2}^2,
    \]
    which implies that (\ref{eq:UniformBounds}) holds true for $ k=1 $.
    Now assume that (\ref{eq:UniformBounds}) holds true for any $ k = 2,3,\cdots ,\ell $.
    To advance the induction, we need to estimate $ \mathbb{E}_{Z^{\ell} }[\mathcal{E} ( \hat{\eta}_{\ell+1})] - \mathcal{E}( \eta^{*}) $.
    In fact, it can be bounded by Proposition \ref{prop:errorDECOM} with $ k = \ell $ as
    \begin{equation}
        \label{eq:BoundErrorInduction}
        \left( \left\| \omega_{1}^{\ell} ( T_{K} + \lambda I) L_{C}^{1 / 2 } \beta_{\lambda} \right\|_{2} +\left\| L_{C}^{1 / 2 }  ( \beta_{\lambda} - \beta^{*}) \right\|_{2}  \right)^2+ \frac{1}{2} \max_{1\leqslant j\leqslant \ell} \left( \sigma^2 + \mathbb{E}_{Z^{i-1} } [ \mathcal{E} (\hat{\eta}_{i}) ] - \mathcal{E}( \eta^{*}) \right).
    \end{equation}
    Recall that $ L_{C}^{1 / 2 } \beta_{\lambda} = (\lambda I + T_{K})^{-1} T_{K} L_{C}^{1 / 2 } \beta^{*} $ in (\ref{eq:betalambdaExpicit}), then we have
    \[
        \begin{aligned}
            \left\| \omega_{1}^{\ell} ( T_{K} + \lambda I) L_{C}^{1 / 2 } \beta_{\lambda} \right\|_{2} & \leqslant \left\| \omega_{1}^{\ell} ( T_{K} + \lambda I) \right\| \left\| ( \lambda I + T_{K} )^{-1} T_{K}  \right\| \left\|  L_{C}^{1 / 2 }  \right\| \left\| \beta^{*} \right\|_{2} \\
            &\leqslant \kappa_2 \left\| \beta^{*} \right\|_{2},
        \end{aligned}
    \]
    where the last step is due to $ \| \omega_{1}^{\ell} ( T_{K} + \lambda I) \| \leqslant \prod_{j=1}^{\ell} ( 1 - \gamma_{j} \lambda ) \leqslant 1 $ and $ \left\| ( \lambda I + T_{K} )^{-1} T_{K}  \right\| \leqslant 1 $.
    Analogously, applying (\ref{eq:DECOMbetaLambdaStar}) yields that
    \[
        \left\| L_{C}^{1 / 2 }  ( \beta_{\lambda} - \beta^{*}) \right\|_{2} = \lambda \left\| \left( \lambda I + T_{K} \right)^{-1} L_{C}^{1 / 2 } \beta^{*}  \right\|_{2} \leqslant \lambda \left\| \left( \lambda I + T_{K} \right)^{-1} \right\| \left\| L_{C}^{1 / 2 } \right\| \left\|  \beta^{*} \right\|_{2} \leqslant \kappa_2 \left\| \beta^{*} \right\|_{2}.
    \]
    Finally, it verifies that
    \[
        \mathbb{E}_{Z^{\ell} }[\mathcal{E} ( \hat{\eta}_{\ell+1})] - \mathcal{E}( \eta^{*}) \leqslant 8 \kappa_{2}^2 \left\| \beta^{*} \right\|_{2}^2 + \sigma^2,
    \]
    by putting all estimates above and the induction assumption into (\ref{eq:BoundErrorInduction}).
    The proof is completed.

\subsection{Proof of Lemma \ref{lem:seriesgamma2}}

    In this subsection, we shall prove Lemma \ref{lem:seriesgamma2} according to different ranges of $ \mu $ and $ \nu $.
    To this end, first we denote by $ \mathcal{S}^{\mathsf{ol}} $ the series in (\ref{eq:seriesBound}), which can decomposed as
    \begin{equation}
        \label{eq:prop31series}
        \begin{aligned}
            \mathcal{S}^{\mathsf{ol}} &= \gamma_1^2 k^{-2\mu} + \gamma_{1}^2 \sum_{i=1}^{k-1} \frac{ i^{-2\mu}  \exp \left\{ - \lambda \sum_{j=i+1}^{k} \gamma_{j} \right\}}{  1 + \left( \sum_{j=i+1}^{k} \gamma_{j} \right)^{\nu} }.
        \end{aligned}
    \end{equation}
    Since $ \gamma_{i} = \gamma_1 i^{-\mu} $ with $ 0 < \mu <1 $ for any $ i \in [k-1] $, we have
    \begin{equation*}
        \sum_{j=i+1}^{k} \gamma_{j} = \gamma_1 \sum_{j=i+1}^{k} j^{-\mu} \geqslant \frac{\gamma_1}{1-\mu} \left[ ( k+1)^{1-\mu} - (i +1)^{1-\mu}  \right],
    \end{equation*}
    and
    \[
        \frac{ \exp \left\{ - \lambda \sum_{j=i+1}^{k} \gamma_{j} \right\}}{  1 + \left( \sum_{j=i+1}^{k} \gamma_{j} \right)^{\nu} } \leqslant \frac{ \exp \left\{ - \frac{\lambda \gamma_1}{1-\mu} \left[ ( k+1)^{1-\mu} - (i +1)^{1-\mu}  \right] \right\} }{ 1 +  \left( \frac{\gamma_1}{1-\mu} \right)^{\nu}  \left[ ( k+1)^{1-\mu} - (i +1)^{1-\mu}  \right]^{\nu} }.
    \]
    The series on the right-hand side of (\ref{eq:prop31series}) can be decomposed as
    \begin{equation}
        \label{eq:olSeriesDECOM}
        \begin{aligned}
            \sum_{i=1}^{k-1} \frac{ i^{-2\mu}  \exp \left\{ - \lambda \sum_{j=i+1}^{k} \gamma_{j} \right\}}{  1 + \left( \sum_{j=i+1}^{k} \gamma_{j} \right)^{\nu} } & \leqslant  \sum_{i\leqslant \frac{k-1}{2}} \frac{ i^{-2\mu} \exp \left\{ -  \frac{\lambda \gamma_1}{1-\mu} \left[ ( k+1)^{1-\mu} - (i +1)^{1-\mu}  \right] \right\} }{ 1 +  \left( \frac{\gamma_1}{1-\mu} \right)^{\nu}  \left[ ( k+1)^{1-\mu} - (i +1)^{1-\mu}  \right]^{\nu} } \\
            & \quad + \sum_{i>\frac{k-1}{2}}^{k-1} \frac{ i^{-2\mu} }{ 1 +  \left( \frac{\gamma_1}{1-\mu} \right)^{\nu}  \left[ ( k+1)^{1-\mu} - (i +1)^{1-\mu}  \right]^{\nu} } \\
            & =: \mathcal{S}_1^{\mathsf{ol}} + \mathcal{S}_2^{\mathsf{ol}}.
        \end{aligned}
    \end{equation}
    In the following, we bound $ \mathcal{S}_1^{\mathsf{ol}} $ and $ \mathcal{S}_2^{\mathsf{ol}} $ for the three cases in (\ref{eq:seriesBound}), respectively.

    First, when $ i \leqslant \frac{k-1}{2} $, we have $ ( k+1)^{1-\mu} - (i +1)^{1-\mu} \geqslant (1-2^{\mu-1} ) (k+1)^{1-\mu} $.
    Therefore,
    \begin{equation*}
        \mathcal{S}_1^{\mathsf{ol}} \leqslant \frac{ \exp \left\{ - \lambda \gamma_1 d_{\mu} (k+1)^{1-\mu} \right\} }{ 1 + \gamma_1^{\nu} d_{\mu}^{\nu} (k+1)^{\nu(1-\mu)} } \sum_{i\leqslant \frac{k-1}{2}} i^{-2\mu}.
    \end{equation*}
    Note that
    \[
        \sum_{i\leqslant \frac{k-1}{2}} i^{-2\mu} \leqslant \int_{0}^{\frac{k-1}{2}} x^{-2\mu} dx \leqslant 1 + \int_{1}^{\frac{k-1}{2}} x^{-2\mu} dx \leqslant
        \begin{cases}
            \frac{(k+1)^{1-2\mu}}{1-2\mu} , & \text{if } 0 < \mu <1 / 2 , \\
            \log(\frac{ \mathrm{e} k}{2}), & \text{if } \mu=1 / 2 , \\
            \frac{2}{2\mu-1}, & \text{if } 1 / 2  < \mu <1.
        \end{cases}
    \]
    Then, we can obtain the corresponding upper bounds for $ \mathcal{S}_1^{\mathsf{ol}} $ for different cases of $ \mu $ and $ \nu $.
    \begin{description}[leftmargin = 1em]
        \item[Case 1: $ \nu = 1, 0<\mu<1 $.]
            \begin{equation*}
                \begin{aligned}
                    \mathcal{S}_1^{\mathsf{ol}}  \leqslant & C_{\mu,1} \gamma_1^{-1} \exp \left\{ - \lambda \gamma_1 d_{\mu} k^{1-\mu} \right\} (k+1)^{- \min \{ \mu ,1-\mu \} } \log(k+1),
                \end{aligned}
            \end{equation*}
            where
            \[
                C_{\mu,1} =
                \begin{cases}
                    4, & \text{if } \mu =1 / 2 , \\
                    \frac{4}{|2\mu-1|(1-2^{\mu-1} )}, & \text{if }  \mu \neq 0,1 / 2 .
                \end{cases}
            \]
        \item[Case 2: $ 0 < \nu,\mu <1 $.]
            \begin{equation*}
                \begin{aligned}
                    \mathcal{S}_1^{\mathsf{ol}} & \leqslant \gamma_1^{-\nu} d_{\mu}^{-\nu} \exp \left\{ - \lambda \gamma_1 d_{\mu} k^{1-\mu} \right\} (k+1)^{-\nu(1-\mu)}
                    \begin{cases}
                        \frac{(k+1)^{1-2\mu}}{1-2\mu} , & \text{if } 0 < \mu <1 / 2 , \\
                        \log(\frac{ \mathrm{e} k}{2}), & \text{if } \mu=1 / 2 , \\
                        \frac{2}{2\mu-1}, & \text{if } 1 / 2  < \mu <1,
                    \end{cases} \\
                    & \leqslant \hat{C}_{\mu,1} \gamma_1^{-\nu} \exp \left\{ - \lambda \gamma_1 d_{\mu} k^{1-\mu} \right\}
                    \begin{cases}
                        (k+1)^{- \mu + (1-\nu)(1-\mu) }, & \text{if } 0 < \mu <1 / 2 , \\
                        (k+1)^{-\nu(1-\mu) } \log(k+1), & \text{if } 1 / 2  \leqslant  \mu <1,
                    \end{cases} \\
                \end{aligned}
            \end{equation*}
            where
            \[
                \hat{C}_{\mu,1} =
                \begin{cases}
                    2^{1+\nu} , & \text{if } \mu =1 / 2 , \\
                    \frac{2 }{|2\mu-1| d_{\mu}^{\nu} }, & \text{if }  \mu \neq 0,1 / 2 .
                \end{cases}
            \]
        \item[Case 3: $ \nu = 0, 1 / 2  < \mu < 1 $.]
        \begin{equation*}
            \begin{aligned}
                \mathcal{S}_1^{\mathsf{ol}} & \leqslant \exp \left\{ - \lambda \gamma_1 d_{\mu} k^{1-\mu} \right\} \sum_{i\leqslant \frac{k-1}{2}}  i^{-2\mu} \leqslant  \frac{2}{2\mu-1}  \exp \left\{ - \lambda \gamma_1 d_{\mu} (k+1)^{1-\mu} \right\}.
            \end{aligned}
        \end{equation*}
        \item[Case 4: $ \nu >1,0 < \mu < 1 $.]
            \begin{equation*}
                \begin{aligned}
                    \mathcal{S}_1^{\mathsf{ol}} \leqslant & \gamma_1^{-\nu} d_{\mu}^{-\nu} \exp \left\{ - \lambda \gamma_1 d_{\mu} k^{1-\mu} \right\} k^{-\nu(1-\mu)}
                    \begin{cases}
                        \frac{k^{1-2\mu}}{1-2\mu} , & \text{if } 0 < \mu <1 / 2 , \\
                        \log(\frac{ \mathrm{e} k}{2}), & \text{if } \mu=1 / 2 , \\
                        \frac{2}{2\mu-1}, & \text{if } 1 / 2  < \mu <1,
                    \end{cases} \\
                    = & \gamma_1^{-\nu} d_{\mu}^{-\nu} \exp \left\{ - \lambda \gamma_1 d_{\mu} k^{1-\mu} \right\}
                    \begin{cases}
                        \frac{1}{1-2\mu} k^{1-2\mu - \nu(1-\mu)}, & \text{if } 0 < \mu <1 / 2 , \\
                        k ^{-\frac{\nu}{2}} \log(\frac{ \mathrm{e} k}{2}), & \text{if } \mu=1 / 2 , \\
                        \frac{2}{2\mu-1} k^{-\nu(1-\mu)} , & \text{if } 1 / 2  < \mu <1,
                    \end{cases} \\
                    \leqslant & \gamma_1^{-\nu} d_{\mu}^{-\nu} \exp \left\{ - \lambda \gamma_1 d_{\mu} k^{1-\mu} \right\}
                    \begin{cases}
                        \frac{1}{1-2\mu} k^{-\mu}, & \text{if } 0 < \mu <1 / 2 , \\
                        (\frac{\mathrm{e}}{2})^{\frac{\nu-3}{2}} \frac{1}{\nu-1}  k^{-\frac{1}{2}}, & \text{if } \mu=1 / 2 , \\
                        \frac{2}{2\mu-1} k^{-\nu(1-\mu)} , & \text{if } 1 / 2  < \mu <1,
                    \end{cases} \\
                \end{aligned}
            \end{equation*}
            since $ 1-2\mu - \nu(1-\mu) = - \mu + (1-\nu)(1-\mu) < -\mu $ and $ \max\limits_{1\leqslant k < \infty} k^{-a} \log(\frac{\mathrm{e}k}{2}) \leqslant (\frac{\mathrm{e}}{2} )^{a} \frac{1}{\mathrm{e}a} $ for $ a = $ $ \frac{\nu-1}{2} $.
            Hence,
            \[
                \begin{aligned}
                    \mathcal{S}_1^{\mathsf{ol}} & \leqslant \overline{C}_{\mu,1} \gamma_1^{-\nu} \exp \left\{ - \lambda \gamma_1 d_{\mu} (k+1)^{1-\mu} \right\}
                    \begin{cases}
                        k^{-\mu}, & \text{if } 0 < \mu \leqslant 1 / 2 , \\
                        k^{-\nu(1-\mu)} , & \text{if } 1 / 2  < \mu < 1,
                    \end{cases} \\
                    & \leqslant \overline{C}_{\mu,1} \gamma_1^{-\nu} \exp \left\{ - \lambda \gamma_1 d_{\mu} (k+1)^{1-\mu} \right\} k^{-\min \{ \mu, \nu(1-\mu) \}},
                \end{aligned}
            \]
            where
            \[
                \overline{C}_{\mu,1} =
                \begin{cases}
                    (\frac{\mathrm{e}}{2})^{\frac{\nu-3}{2}} \frac{1}{(\nu-1)d_{\mu}^{\nu}} , & \text{if } \mu =1 / 2 , \\
                    \frac{2 }{|2\mu-1| d_{\mu}^{\nu} }, & \text{if }  \mu \neq 0,1 / 2.
                \end{cases}
            \]
    \end{description}

    Second, when $ i > \frac{k-1}{2} $, note that $ i^{-\mu} \leqslant (\frac{i+2}{3})^{-\mu} \leqslant 3 (i+2 )^{-\mu}  $ for any $ i \geqslant 1 $.
    Further, for any $ x \in [i,i+1] $, there holds $ i^{-\mu} \leqslant 3(x+1)^{-\mu}  $ and $  (k+1)^{1-\mu} - (x+1)^{1-\mu} \leqslant (k+1)^{1-\mu} - (i+1)^{1-\mu}  $.
    Therefore, we can bound $ \mathcal{S}_2^{\mathsf{ol}} $ by transforming it to integral
    \begin{equation*}
        \begin{aligned}
            \mathcal{S}_2^{\mathsf{ol}} & \leqslant 9 \sum_{i>\frac{k-1}{2}}^{k-1} \int_{i}^{i+1} \frac{  (x+1)^{-2\mu} }{ 1 +  \left( \frac{\gamma_1}{1-\mu} \right)^{\nu}  \left[ ( k+1)^{1-\mu} - (x +1)^{1-\mu}  \right]^{\nu}  }dx \\
            & \leqslant 18 (k+1)^{-\mu}  \int_{\frac{k-1}{2}}^{k} \frac{  (x+1)^{-\mu} }{ 1 +  \left( \frac{\gamma_1}{1-\mu} \right)^{\nu}  \left[ ( k+1)^{1-\mu} - (x +1)^{1-\mu}  \right]^{\nu}  }dx.
        \end{aligned}
    \end{equation*}
    Using the change of variables $ \xi = ( k+1)^{1-\mu} - (x +1)^{1-\mu} $ with $ \mathrm{d}\xi = -(1-\mu) (x+1)^{-\mu} \mathrm{d}x $, we deduce
    \begin{equation*}
        \begin{aligned}
            \mathcal{S}_2^{\mathsf{ol}} & \leqslant 18 \min \left\{ 1, \left( \frac{\gamma_1}{1-\mu} \right)^{\nu} \right\} ^{-1}(k+1)^{-\mu} \frac{1}{1-\mu} \int_{0}^{(1-2^{\mu-1}) (k+1)^{1-\mu} } \frac{1}{1+\xi^{\nu} } \mathrm{d}\xi.
        \end{aligned}
    \end{equation*}
    The specific bounds for $ \mathcal{S}_2^{\mathsf{ol}} $ are given as follows.
    \begin{description}[leftmargin = 1em]
        \item[Case 1: $ \nu = 1, 0<\mu<1 $.]
            \begin{equation*}
                \begin{aligned}
                    \mathcal{S}_2^{\mathsf{ol}} \leqslant & 18 \min \left\{ 1,  \frac{\gamma_1}{1-\mu}  \right\} ^{-1}(k+1)^{-\mu} \frac{1}{1-\mu} \log \left( 1 + (1-2^{\mu-1}) (k+1)^{1-\mu} \right) \\
                    \leqslant & C_{\mu,2} k^{-\mu} \log(k+1),
                \end{aligned}
            \end{equation*}
            where $ C_{\mu,2} = 18 \min \left\{ 1-\mu, \gamma_1\right\} ^{-1} \left( 1-\mu + \log(2-2^{\mu-1} ) \right).$
        \item[Case 2: $ 0 < \nu,\mu <1 $.]
            \begin{equation*}
                \begin{aligned}
                    \mathcal{S}_2^{\mathsf{ol}} \leqslant & 18 \min \left\{ 1, \left( \frac{\gamma_1}{1-\mu} \right)^{\nu} \right\} ^{-1}(k+1)^{-\mu} \frac{1}{1-\mu} \left( 1+\int_{1}^{(1-2^{\mu-1}) (k+1)^{1-\mu} } \xi^{-\nu}  \mathrm{d}\xi \right) \\
                    \leqslant & \hat{C}_{\mu,2} (k+1)^{-\mu + (1-\mu)(1-\nu)},
                \end{aligned}
            \end{equation*}
            where $ \hat{C}_{\mu,2} = 18 \min \left\{ 1, \left( \frac{\gamma_1}{1-\mu} \right)^{\nu} \right\} ^{-1} \frac{(1-2^{\mu-1} )^{1-\nu} }{(1-\mu)(1-\nu)}. $
        \item[Case 3: $ \nu = 0, 1 / 2  < \mu < 1 $.]
            \begin{equation*}
                \begin{aligned}
                    \mathcal{S}_2^{\mathsf{ol}} & \leqslant 18 d_{\mu} \min \left\{ 1, \left( \frac{\gamma_1}{1-\mu} \right)^{\nu} \right\} ^{-1} (k+1)^{1-2\mu} .
                \end{aligned}
            \end{equation*}
        \item[Case 4: $ \nu >1,0 < \mu < 1 $.]
            \[
                \begin{aligned}
                    \mathcal{S}_2^{\mathsf{ol}} \leqslant & 18 \min \left\{ 1, \left( \frac{\gamma_1}{1-\mu} \right)^{\nu} \right\}^{-1} (k+1)^{-\mu} \frac{1}{1-\mu} \left( 1+\int_{1}^{(1-2^{\mu-1}) (k+1)^{1-\mu} } \xi^{-\nu}  \mathrm{d}\xi \right) \\
                    \leqslant & 18 \min \left\{ 1, \left( \frac{\gamma_1}{1-\mu} \right)^{\nu} \right\}^{-1} \frac{\nu}{(1-\mu)(\nu-1)} k^{-\mu},
                \end{aligned}
            \]
    \end{description}

    Finally, combining the bounds for $ \mathcal{S}_{1}^{\mathsf{ol}} $ and $ \mathcal{S}_2^{\mathsf{ol}} $ with (\ref{eq:prop31series}) and (\ref{eq:olSeriesDECOM}), we have
    \begin{description}[leftmargin = 1em]
        \item[Case 1: $ \nu = 1, 0<\mu<1 $.]
            \begin{equation*}
                \begin{aligned}
                    \mathcal{S}^{\mathsf{ol}} \leqslant & C_{\mu} \gamma_1 \left(  \exp \left\{ - \lambda \gamma_1 d_{\mu} k^{1-\mu} \right\} k^{- \min \{ \mu ,1-\mu \} } + k^{-\mu}  \right)\log(k+1), \\
                \end{aligned}
            \end{equation*}
            where $ C_{\mu} = 1+ C_{\mu,1} +C_{\mu,2}.  $
        \item[Case 2.1: $ 0 < \nu <1,0<\mu<1 /2 $.]
            \begin{equation*}
                \begin{aligned}
                    \mathcal{S}^{\mathsf{ol}} \leqslant & \hat{C}_{\mu} \gamma_1 (k+1)^{- \mu + (1-\nu)(1-\mu) },
                \end{aligned}
            \end{equation*}
            where $ \hat{C}_{\mu} = 1 + \hat{C}_{\mu,1} + \hat{C}_{\mu,2}. $
        \item[Case 2.2: $ 0 < \nu <1,1 /2\leqslant \mu<1 $.]
            \begin{equation*}
                \mathcal{S}^{\mathsf{ol}} \leqslant \hat{C}_{\mu} \gamma_1 \left[ (k+1)^{-\mu + (1-\mu)(1-\nu)} +\exp \left\{ - \lambda \gamma_1 d_{\mu} k^{1-\mu} \right\}k^{-\nu(1-\mu) } \log(k+1) \right].
            \end{equation*}
        \item[Case 3: $ \nu = 0, 1 / 2  < \mu < 1 $.]
            \begin{equation*}
                \begin{aligned}
                    \mathcal{S}^{\mathsf{ol}} & \leqslant \widetilde{C}_{\mu} \gamma_1 \left( \exp \left\{ - \lambda \gamma_1 d_{\mu} k^{1-\mu} \right\} + k^{1-2\mu} \right),
                \end{aligned}
            \end{equation*}
            where $ \widetilde{C}_{\mu} = \left( 1 + \frac{2}{2\mu-1} + 18 d_{\mu} \min \left\{ 1, \left( \frac{\gamma_1}{1-\mu} \right)^{\nu} \right\} ^{-1} \right) $.
        \item[Case 4: $ \nu >1,0 < \mu < 1 $.]
            \[
                \begin{aligned}
                    \mathcal{S}^{\mathsf{ol}} & \leqslant \overline{C}_{\mu} \gamma_1 \left( \exp \left\{ - \lambda \gamma_1 d_{\mu} (k+1)^{1-\mu} \right\} k^{-\min \{ \mu, \nu(1-\mu) \}} + k^{-\mu}  \right) , \\
                \end{aligned}
            \]
            where $ \overline{C}_{\mu} = 1+ \overline{C}_{\mu,1} + 18 \min \left\{ 1, \left( \frac{\gamma_1}{1-\mu} \right)^{\nu} \right\}^{-1} \frac{\nu}{(1-\mu)(\nu-1)}.  $
    \end{description}
    The proof is completed.

\end{document}